\documentclass{applemlr}

\usepackage{amsmath}
\usepackage{enumerate}
\usepackage{algorithm}
\usepackage{algpseudocode}
\usepackage{amsfonts}
\usepackage{amsthm}
\usepackage{cleveref}
\usepackage{diagbox}
\usepackage{colortbl}
\usepackage{amssymb}
\usepackage{xspace}
\usepackage{wrapfig}
\usepackage{adjustbox}
\usepackage{tabularx}
\usepackage{booktabs}
\usepackage{mathtools}
\usepackage{tikz}
\usepackage{enumitem}
\usepackage{silence}
\usepackage{dsfont}
\usepackage[table]{xcolor}
\usepackage[dvipsnames]{xcolor}
\usepackage{multirow}
\usepackage{makecell}
\usepackage{xfakebold}
\usepackage{amsmath,amsfonts,bm}

\def\eqref#1{equation~\ref{#1}}
\def\1{\bm{1}}

\DeclareMathAlphabet{\mathsfit}{\encodingdefault}{\sfdefault}{m}{sl}
\SetMathAlphabet{\mathsfit}{bold}{\encodingdefault}{\sfdefault}{bx}{n}

\newcommand{\E}{\mathbb{E}}

\newcommand{\R}{\mathbb{R}}

\newcommand{\softmax}{\mathrm{softmax}}

\definecolor{textgray}{HTML}{6E6E73}
\usetikzlibrary{positioning, calc}
\usetikzlibrary{decorations.pathmorphing}

\makeatletter
\patchcmd{\wrong@fontshape}{\@gobbletwo}{}{}{}
\makeatother
\numberwithin{equation}{section}
\makeatletter
\AtBeginDocument{
  \urlstyle{sf}
  
}
\makeatother

\definecolor{light}{RGB}{125, 125, 125}
\crefname{tcb@cnt@pbox}{code}{code}
\Crefname{tcb@cnt@pbox}{Code}{Code}
\crefname{assumption}{assumption}{assumption}
\Crefname{assumption}{Assumption}{Assumptions}

\newtcolorbox[auto counter]{pbox}[2][]{
  colback=white,
  title=Code~\thetcbcounter: #2,
  #1,fonttitle=\sffamily,
  fontupper=\sffamily,
  arc=2pt,
  colframe=bgcolor,
  coltitle=fgcolor,
  colbacktitle=bgcolor,
  toptitle=0.25cm,
  bottomtitle=0.125cm
}

\makeatletter
\newcommand\applefootnote[1]{%
  \begingroup
  \renewcommand\thefootnote{}%
  \renewcommand\@makefntext[1]{\noindent##1}%
  \footnote{#1}%
  \addtocounter{footnote}{-1}%
  \endgroup
}
\makeatother

\definecolor{cverbbg}{gray}{0.90}

\usepackage[utf8]{inputenc}
\usepackage[T1]{fontenc}
\usepackage{hyperref}
\usepackage{url}
\usepackage{booktabs}
\usepackage{amsfonts}
\usepackage{amsmath}
\usepackage{amsthm}
\usepackage{amssymb}
\usepackage{nicefrac}
\usepackage{xcolor}
\usepackage{enumitem}
\usepackage{multirow}
\usepackage{graphicx}
\usepackage{wrapfig}
\usepackage{colortbl}
\usepackage{adjustbox}
\usepackage{multirow}
\usepackage{graphicx}    % for \rotatebox in multirow
\usepackage{subcaption}
\usepackage{pifont}
\usepackage{lmodern} 

\usepackage{amsthm}
\newtheorem{theorem}{Theorem}[section]

\newtheorem{proposition}[theorem]{Proposition}

\theoremstyle{definition}
\newtheorem{definition}[theorem]{Definition}
\theoremstyle{remark}

\newcommand{\memblock}{\textsc{MemoryBlock}}
\newcommand{\embmem}{\textsc{EmbeddingMemory}}

\newcommand{\V}{\mathcal{V}}

\newcommand{\bzero}{\mathbf{0}}
\newcommand{\balpha}{\boldsymbol{\alpha}}
\newcommand{\RMSNorm}{\operatorname{RMSNorm}}
\newcommand{\FFN}{\operatorname{FFN}}
\newcommand{\Attn}{\operatorname{Attn}}
\newcommand{\norm}[1]{\left\lVert #1 \right\rVert}
\newcommand{\tide}{\textsc{TIDE}}
\newcommand{\MB}{\textsc{memoryblock}}

\newcommand{\cL}{\mathcal{L}}

\usepackage{booktabs,xcolor,colortbl,array}
\definecolor{rarerow}{RGB}{237,231,252}
\definecolor{midrow}{RGB}{245,245,245}
\definecolor{commonrow}{RGB}{220,242,233}
\definecolor{rarecol}{RGB}{83,74,183}
\definecolor{midcol}{RGB}{80,80,80}
\definecolor{commoncol}{RGB}{15,110,86}
\definecolor{lightgray}{gray}{0.85}

\title{TIDE: Every Layer Knows the Token Beneath the Context }

\author{Ajay Jaiswal}
\author{Lauren Hannah}
\author{Han-Byul Kim}
\author{Duc Hoang}
\author{Mehrdad Farajtabar}
\author{Minsik Cho}

\affiliation{Apple}

\abstract{
We revisit a universally accepted but under-examined design choice in every modern LLM: a token index is looked up once at the input embedding layer and then permanently discarded. This \emph{single-injection assumption} induces two structural failures: (i) the \textit{Rare Token Problem}, where a Zipf-type distribution of vocabulary causes rare-token embeddings are chronically under-trained due to receiving a fraction of the cumulative gradient signal compared to common tokens; and (ii) the \textit{Contextual Collapse Problem}, where limited parameters models map distributionally similar tokens to indistinguishable hidden states. As an attempt to address both, we propose \tide{}, which augments the standard transformer with \embmem{}: an ensemble of $K$ independent \memblock{}s that map token indices to context-free semantic vectors, computed once and injected into every layer through a depth-conditioned softmax router with a learnable null bank. We \emph{theoretically} and \emph{empirically} establish the benefits of \tide{} in addressing the issues associated with single-token identity injection as well as \emph{improve performance} across multiple language modeling and downstream tasks. }

\metadata[Correspondence]{\sffamily Ajay Jaiswal: \url{ajaiswal23@apple.com}}
\date{\sffamily\today}

\begin{document}

\maketitle

\section{Introduction}
\label{sec:introduction}

Scaling modern large language models (LLMs) involves devoting substantial representational capacity towards \textit{contextualizing} tokens through innovating attention mechanisms, enlarging feed-forward modules, and stacking deep transformer layers. In contrast, a critical LLM component that has been widely overlooked in recent advancements is the \textbf{token index} - \textit{the only piece of information that unambiguously identifies what a token is}. The token index is looked up once at the input embedding layer and then permanently discarded. Every subsequent computation across all $L$ transformer layers operates on a contextualized hidden state that never again directly consults which vocabulary entries are being processed. This single-injection assumption creates \underline{two} distinct failure modes:

\ding{182} \textbf{The Rare Token Problem:} Natural language vocabularies obey power law scaling, specifically Zipf's law \citep{zipf1949human, pilgrim2021bias}: the most frequent 1\% of tokens account for $\sim\!80\%$ of corpus occurrences. Under SGD, cumulative gradient signal for each token embedding is proportional to its frequency (Section \ref{sec:grad_starvation}), leaving rare-token embeddings (e.g. rare named entities, technical terms, low-frequency morphological forms) persistently under-trained (Figure \ref{fig:rare-token-embeddings}).

\ding{183} \textbf{Contextual Hidden State Collapse:} During training, FFNs are forced into \emph{representational overloading} where they simultaneously implement structural transformations of the residual stream and serve as the primary store of token-specific factual knowledge \citep{meng2022rome, dai2022knowledgeneurons}. The token index is never re-consulted at intermediate layers, and the only mechanism the FFNs have to differentiate two tokens at depth relies on contextual mixture of residual and attention output. However, in case when two semantically distinct tokens appear in nearly identical syntactic environments, the context provides limited differentiating signal and their hidden states become nearly indistinguishable across the network (Figure \ref{fig:collapse}).

Motivated by these challenges, we pose a critical question: \textit{How can we provide every transformer layer with persistent, token-identity-conditioned knowledge, independent of and complementary to the contextual residual stream?} Unlike prior approaches that focus on post-hoc analysis of {\it de facto} FFNs \citep{geva2022vocabspace,meng2022rome,meng2023memit} or retrofit external retrieval at inference time \citep{lewis2020rag,borgeaud2022retro,izacard2023atlas}, we adopt an alternative approach:
designing and training from scratch a novel transformer architecture that maintains
a dedicated semantic memory indexed directly by static token identity information.
 
In this work, we propose \tide{} (\textbf{T}oken \textbf{I}dentity \textbf{D}elivered \textbf{E}verywhere), an architectural modification to standard transformer that maintains a dedicated semantic memory indexed directly by token identity (Figure \ref{fig:main_diagram}). \tide{}  introduces \embmem{}, an ensemble of $K$ independent \memblock{}s each mapping token indices to static and context-free learned semantic vectors that can be injected to each transformer layer with a persistent, token-conditioned signal in parallel to the contextual residual stream. Our key contributions can be summarized as:

\begin{itemize}
\item \textbf{Architectural.} \tide{} introduces a \emph{token-level unified embedding memory} that enables  $K$ disjoint pathways for token-level gradient accumulation. The tensor from memory embeddings tensor is computed \emph{once per forward pass} and injected into \emph{every} transformer layer via a per-layer softmax routing mechanism conditioned on the post-attention hidden state.

\item \textbf{Theoretical.} We formalize the two failure mode in standard transformer and prove that \tide{} (i) asymptotically generalizes the standard transformer; (ii) amplifies the per-token cumulative gradient signal by a factor of $K$, and (iii) routes around the FFN's Lipschitz constraint by exposing a discrete, token-indexed input with no obligation to hidden states.

\item \textbf{Empirical.} We empirically validate that \tide{} significantly benefits \emph{rare} tokens and mitigates contextual collapse problem. Across model scales from $350$M to $1$B parameters, \tide{} consistently delivers up to performance improvements over standard transformer across various language modeling datasets (\emph{e.g.,} Wikitext, PubMed, DCLM) as well as downstream tasks (\emph{e.g.,} HellaSwag, ARC, PIQA).

\end{itemize}

\begin{figure}
    \centering
    \includegraphics[width=0.99\linewidth]{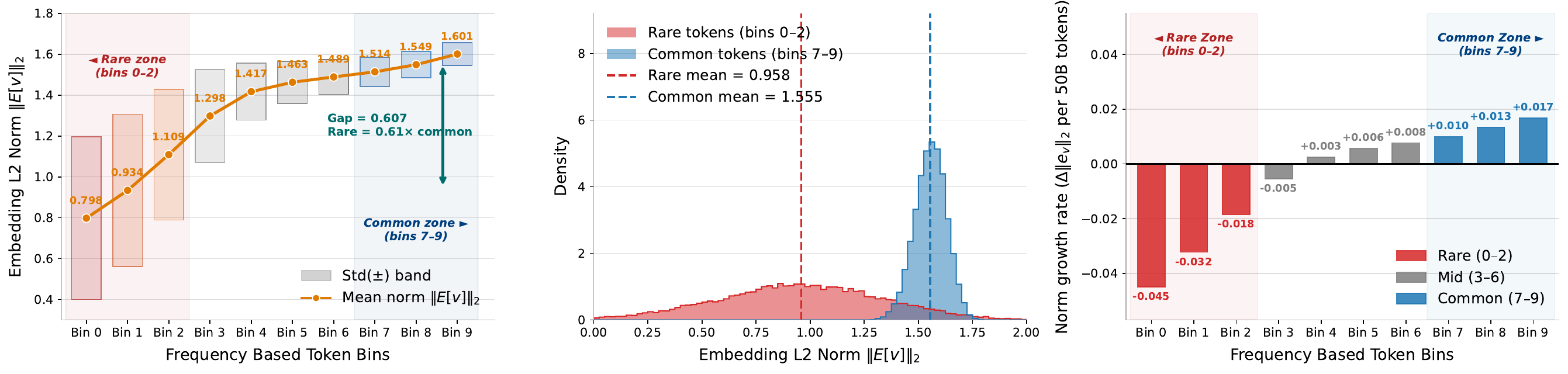}
    \caption{\textbf{Empirical Evidence that Rare Token Embeddings Remain Under-trained:} (a) Mean embedding $l2$-norm of LLaMa-\texttt{Base}-1B pretrained checkpoint showing a \underline{monotonic increase} in norm from rare to common bins; (b) Embedding norm distributions for rare and common tokens: existence of \textit{wide \underline{rare} distribution versus the \underline{narrow} common peak} confirms that rare embeddings remain noise-dominated and under-trained; (c) Bin-wise norm growth rate across intermediate training checkpoints per 50 billion tokens: rare token norms exhibit \textit{\underline{monotonic decline}} with continued training while common token continue growing.}
    \label{fig:rare-token-embeddings}
\end{figure}

\section{When Context is Not Enough: Diagnosing Standard Transformers}
\label{sec:single_token_injection_problem}
\subsection{The Rare Token Problem.}

\paragraph{Gradient Starvation Bottleneck: }\label{sec:grad_starvation}
Under minibatch SGD with batch size $B$, sequence length $T$, and per-token squared gradient norm bounded by $G^{2}$, the embedding $e_v \in \R^{d}$ for token $v$ receives a non-zero gradient only when $v$ appears in the current batch. In this setting, the expected cumulative squared gradient norm after $\tau$ training steps satisfies:
\begin{equation}
  \E\!\left[\sum_{s=1}^{\tau}\norm{\nabla_{e_v}\mathcal{L}_s}^{2}\right]
  \;\leq\; \tau \cdot f_v \cdot B \cdot T \cdot G^{2},
  \label{eq:grad_starvation}
\end{equation}
where $f_v := \Pr[\text{uniformly drawn token position equals } v] \in (0,1)$ is the unigram probability of $v$, with $\sum_{v \in \mathcal{V}} f_v = 1$. Token $v$ is \emph{rare} if $f_v = \epsilon$ for some $\epsilon \ll 1/(BT)$, and token $u$ is \emph{common} if $f_u \geq c$ for some constant $c > 0$ independent of $|\mathcal{V}|$. The full derivation of~\eqref{eq:grad_starvation} is given in Appendix~\ref{app:grad_starvation_proof}.

\begin{table}[h]
  \centering
  \setlength{\tabcolsep}{5pt}   
  \begin{tabular}{lcccc}
    \toprule
    \textbf{Tier} & \textbf{Bin(s)} &
    \textbf{Count} & $f_{v}$ & $\E[N_{v}]$ \\
    \midrule
    \rowcolor{rarerow}
    Hapax (rarest)   & 0    & 1
      & $8.3\!\times\!10^{-9}$  & $\approx$\textcolor{red}{1,660}  \\
    \rowcolor{rarerow}
    Near-hapax       & 1    & $\sim$4
      & $3.3\!\times\!10^{-8}$  & $\approx$6,640  \\
    \rowcolor{rarerow}
    Uncommon         & 2    & $\sim$10
      & $8.3\!\times\!10^{-8}$  & $\approx$16,600 \\
    \rowcolor{midrow}
    Mid-freq.        & 3--6 & ${\sim}10^{2\text{--}3}$
      & ${\sim}10^{-6}$         & ${\approx}10^{5\text{--}6}$ \\
    \rowcolor{commonrow}
    Common (highest) & 7--9 & ${\sim}10^{6}$
      & $8.3\!\times\!10^{-3}$  & $\approx$\textcolor{blue}{\textbf{$1.66\!\times\!10^{9}$}} \\
    \bottomrule
  \end{tabular}
  \caption{Expected non-zero gradient updates $\E[N_{v}]$ of token bins with 200B training tokens.}
  \label{tab:wikitext103}
\end{table}

In an example corpus of Wikitext-103 \citep{merity2016pointer} tokenized using LLaMA-3 tokenizer ($|\mathcal{V}|=128{,}256$) to generate frequency bins (Appendix \ref{app:binning}), the gradient disparity between \emph{rare} and \emph{common} tokens becomes severe. Over a training budget of 200B tokens with $B=8$, $T=2048$, the expected number of non-zero gradient updates to token $v$'s embedding is given as:
\begin{equation}
  \E[N_{v}]
  \;=\;
  \tau\bigl(1-(1-f_{v})^{BT}\bigr)
  \;\approx\;
  \tau\cdot f_{v}\cdot BT
  \quad\text{for small }f_{v}\,.
  \label{eq:expected_updates}
\end{equation}
In reference to frequency bins defined in Appendix \ref{app:binning}, Table ~\ref{tab:wikitext103} instantiates this across the $200$B tokens in our training dataset illustrating the existence of high gradient update disparity between rare and common tokens. Additionally, it can empirically inferred from the Figure \ref{fig:rare-token-embeddings}(c) that this disparity doesn't limit itself as a cold-start artifact but grows \emph{monotonically} as the training progresses. The rare tokens' norms decline while common tokens' norms continuously increase.

\textbf{\textit{Ratio of gradient signal for rare and common tokens:}} For rare $v$ ($f_{v}=\varepsilon$) and common $u$ ($f_{u}\geq c>0$), let $G_{\min}^{2}>0$ be a lower bound on the per-step squared gradient norm conditioned on token $u$ appearing in a batch. The ratio of cumulative gradient signals satisfies:
\begin{equation}
  \frac{\E\!\left[\sum_{s}\|\nabla_{e_{v}}\cL_{s}\|^{2}\right]}
       {\E\!\left[\sum_{s}\|\nabla_{e_{u}}\cL_{s}\|^{2}\right]}
  \;\leq\;
  \frac{\varepsilon\, BT\, G^{2}}{\kappa\, G_{\min}^{2}}
  \;=\; O\!\left(\varepsilon/c\right)
  \label{eq:grad_ratio}
\end{equation}
where $\kappa := 1-(1-c)^{BT} > 0$ with $BT$, $G^{2}$, and $G_{\min}^{2}$ as fixed positive constants. The full derivation is given in Appendix~\ref{app:ratio}. For the empirical instantiation in Table \ref{tab:wikitext103}, the ratio between rare tokens (Bin 1) and common tokens (Bin 9) is $\varepsilon/c\approx10^{-6}$, a disparity of six orders of magnitude of gradient signal between rare and common tokens over the same training budget.

\begin{figure}
    \centering
    \includegraphics[width=0.99\linewidth]{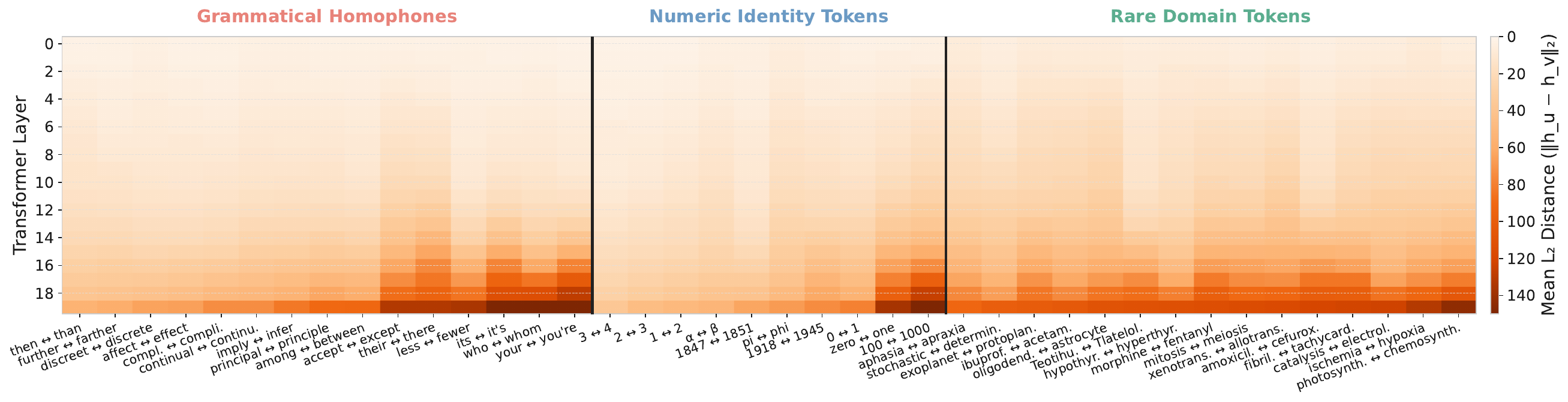}
    \caption{\textbf{Empirical Evidence of Contextual Collapse:} Heatmap illustrating the mean $\ell_2$-distance $\|h^{(\ell)}_u - h^{(\ell)}_v\|$ between hidden states (LLaMa-\texttt{Base}-1B) of token pairs across 250 template sentences from three example categories of contextual collapse. For all sampled pairs, the distance remains \textit{near-zero} for majority of layers (except towards end) confirming the presence of contextual collapse.}
    \label{fig:collapse}
\end{figure}

\subsection{Contextual Collapse and the FFN’s Blind Spot.}
\label{sec:contextual_collapse}
 
As mentioned before, the gradient starvation issue causes the rare-token embeddings to converge to low-norm, noisy representations. More seriously, when two distinct tokens carry poorly trained embeddings of similar magnitude, a deeper structural failure arises: the hidden states produced for those tokens across all transformer layers may become indistinguishable, which can more problematic with similar context shared.   We formalize this failure mode and show that it is an inherent consequence of the Lipschitz continuity imposed on any FFN by its continuous domain.

\paragraph{The Contextual Collapse Phenomenon:} At each layer $\ell$, the hidden state $h^{(\ell)}_v \in \mathbb{R}^d$ of a token $v$ is produced by the attention mechanism operating on the surrounding context.  When two tokens $u \neq v$ appear in nearly identical syntactic environments, such as in case for grammatical homophones (\emph{their} or \emph{there}), numeric identity tokens (\emph{1847}, \emph{1851}, or \emph{1849}), or rare domain-specific synonyms (\emph{ibuprofen} or \emph{acetaminophen}), the context provides no distinguishing signal and thereby attention produces similar outputs for both. 

We formally define this as:

\begin{definition}[Contextual Collapse Set]
For a tolerance $\delta > 0$, the \emph{contextual collapse set} at layer $\ell$ can be formally defined as:
\[
  \mathcal{C}_\delta^{(\ell)}
  \;:=\;
  \bigl\{(u,v)\in\mathcal{V}^2 : u\neq v,\;
  \|h^{(\ell)}_u - h^{(\ell)}_v\|\le\delta\bigr\},
\]
where the hidden states are averaged over a representative corpus of contexts.
\end{definition}

Figure~\ref{fig:collapse} provides direct empirical evidence of contextual collapse in LLaMa-\texttt{Base}-1B standard model estimated using 150 template sentences that differ by a single token pair under consideration. For each of the three example canonical categories, the mean $\ell_2$ distance $\|h^{(\ell)}_u - h^{(\ell)}_v\|$ remains persistently small across the entire depth axis except the last few final layers, confirming the prevalent existence of collapse. Note that this phenomenon is more severe across numerical tokens category having notable collapse (small $\delta$) even within the final layer's hidden states.     

\begin{proposition}[FFN Approximation Lower Bound on Collapsed Tokens]
\label{prop:ffn_lower_bound}
Let $(u,v)\in\mathcal{C}_\delta^{(\ell)}$ be a collapsed token pair and let
$g:\mathcal{V}\to\mathbb{R}^d$ be any target function satisfying
$\|g(u) - g(v)\| = C > 0$.  Then for \emph{any} choice of weights $W_1, W_2$:
\[
  \max\bigl\{\|\mathrm{FFN}(h_u) - g(u)\|,\;\|\mathrm{FFN}(h_v) - g(v)\|\bigr\}
  \;\ge\;
  \frac{C - L_{\mathrm{FFN}}\,\delta}{2}.
\]
When $C > L_{\mathrm{FFN}}\,\delta$, the right-hand side is strictly positive:
the FFN cannot approximate $g$ to arbitrary precision on the collapsed pair
$(u,v)$, regardless of how many parameters it has.
\end{proposition}

\textbf{\textit{Proof sketch.}}
Since $(u,v)\in\mathcal{C}_\delta^{(\ell)}$, the Lipschitz bound forces $\|\mathrm{FFN}(h_u)-\mathrm{FFN}(h_v)\|\le L_{\mathrm{FFN}}\delta$. Applying the triangle inequality to the target separation $C = \|g(u)-g(v)\|$ and substituting this bound yields $\|\mathrm{FFN}(h_u)-g(u)\| + \|\mathrm{FFN}(h_v)-g(v)\| \ge C - L_{\mathrm{FFN}}\delta$. Since the maximum of two non-negative terms is at least half their sum, the result follows. See Appendix~\ref{app:ffn_lower_bound_proof} for details.

In this bound, $\delta$ is determined by the embeddings and attention layers; it is fixed before the FFN acts. The separation target $C$ is determined by the downstream task.The Lipschitz constant $L_{\mathrm{FFN}}$ is the only term the FFN controls, but it is bounded in practice because large $L_{\mathrm{FFN}}$ amplifies \emph{every} input perturbation, degrading performance on the majority of non-collapsed tokens. The bound exposes a structural limitation: given fixed upstream representations, no FFN, regardless of width, can resolve a collapsed token pair without destabilizing other inputs.  The token index is injected once at the embedding layer and never reintroduced; unlike position, which is re-injected via RoPE at every attention layer, token identity has no recovery mechanism. Once intermediate layers erase the distinction, it is permanently lost to all subsequent computation.

\begin{figure}
    \centering
    \includegraphics[width=0.99\linewidth]{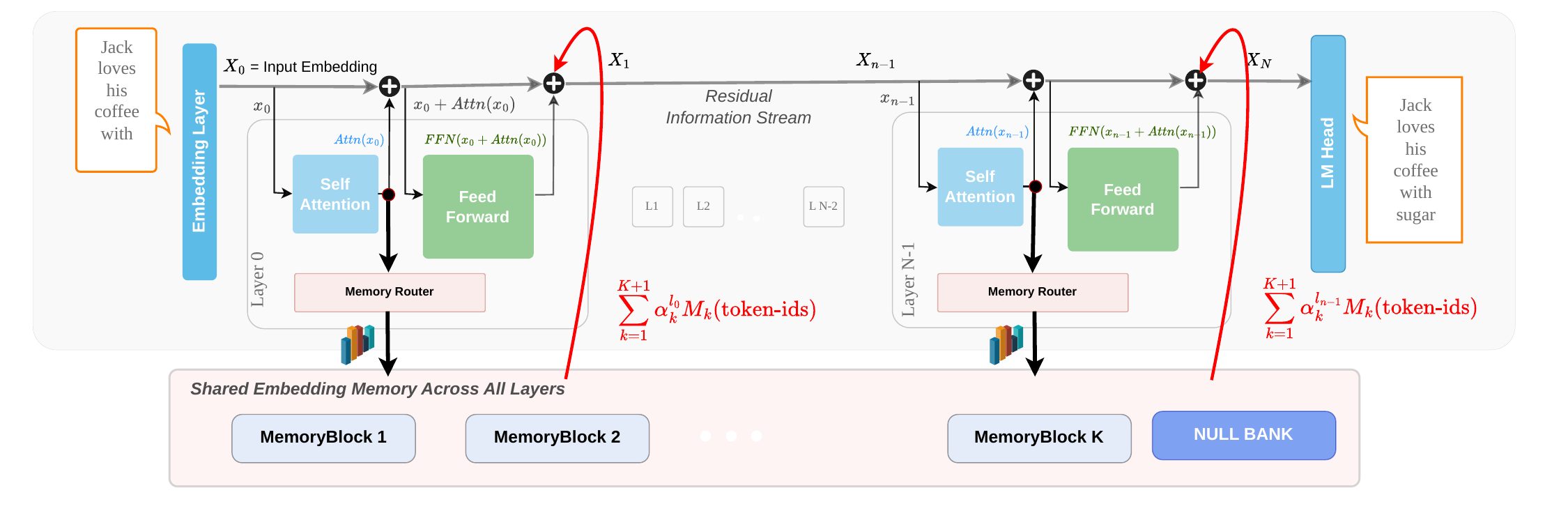}

    \caption{\textbf{Main Architecture Diagram:} \tide{} augments standard transformers with a parallel and globally shared \embmem{} module (red region) consisting of $K$ independent \memblock{}, each mapping raw token indices to \textit{context-free token identity signal}. Each layer uses a linear router to combine memory blocks signals and injects that into the residual stream additively.}

    \label{fig:main_diagram}
\end{figure}

% ── 2. TIDE ───────────────────────────────────────────────────────────────────
\section{TIDE: \underline{T}oken \underline{I}dentity \underline{D}elivered \underline{E}verywhere}
\label{sec:tide}
In section \ref{sec:single_token_injection_problem}, we investigated and formalized two failure mode, \emph{i.e.,} rare token and contextual collapse problem, within the standard transformer architecture. In this work, we address these issues with a novel architecture modification: \tide{} counters the single-injection assumption in conventional design of modern LLMs. \tide{} stops discarding the token identity information after embedding layer and instead make it directly accessible at every depth, so that each layer retains a token-discriminative signal independent of the contextual residual stream. 

% ── 2.1 Preliminaries and Notations ──────────────────────────────────────────

\subsection{Preliminaries and Notations.}
\label{sec:notation}
 
Let $\V$ denote a vocabulary of size $|\V|$, $d$ the model hidden dimension, $d_b$ the \MB{} embedding dimension, $K$ the number of \MB{}s, $L$ the number of transformer layers, $T$ the input sequence length, and $B$ the batch size. We use $x \in \mathbb{Z}^{B \times T}$ for a batch of token index sequences and $h^{(\ell)} \in \R^{B \times T \times d}$ for hidden states at layer $\ell$. The standard LLaMA-style transformer block at layer $\ell$ computes:
\begin{align}
  \tilde{h}^\ell &= h^{\ell-1} + \Attn\!\bigl(\RMSNorm(h^{\ell-1})\bigr),
  \label{eq:attn} \\
  h^\ell &= \tilde{h}^\ell + \FFN\!\bigl(\RMSNorm(\tilde{h}^\ell)\bigr),
  \label{eq:base}
\end{align}
where $\Attn$ is multi-head self-attention with rotary position embeddings and $\FFN$ is a SiLU-gated feed-forward network. The primary embedding table $E \in \R^{|\V| \times d}$ maps each token index to an initial hidden state $h^{(0)} = E[x]$ that will be processed by different transformer blocks.

% ── 2.2 Architecture Design ──────────────────────────────────────────────────
\subsection{TIDE Architecture Design.}
\label{sec:arch}
 
\tide{} augments the standard transformer with a parallel \textit{token-identity memory pathway} composed of three components:
 
\textbf{\MB{}s:}
Each of the $K$ \MB{}s maintains a dedicated embedding table $E_k \in \R^{|\V| \times d_b}$ and maps a token index $v \in \V$ to a $d_b$-dimensional vector via a single embedding lookup followed by RMSNorm \citep{zhang2019rmsnorm}:
\begin{equation}
  M_k(v) = \RMSNorm\!\bigl(E_k[v]\bigr) \in \R^{d_b}.
  \label{eq:memblock}
\end{equation}
Each block maintains its own independent embedding table with no parameter sharing across blocks, encouraging each \MB{} to learn a distinct projection of the token identity space.
 
\textbf{\textsc{EmbeddingMemory} ensemble:}
The $K$ \MB{}s are stacked into a single memory tensor computed \textit{once} per forward pass and shared across all $L$ transformer layers:
\begin{equation}
  \mathbf{M} = \mathrm{Stack}_k\!\bigl(M_k(x)\bigr) \in \R^{B \times T \times K \times d_b}.
  \label{eq:ensemble}
\end{equation}

\textbf{Depth-conditioned router and additive fusion:} Within each transformer block, the post-attention normalised hidden state $\tilde{n}^\ell = \RMSNorm(\tilde{h}^\ell)$ is fed to a lightweight linear router to generate composition ratio $\alpha_k^\ell$ corresponding to $k$-th memory block. We additionally introduce a \textbf{null bank} at slot $K{+}1$ satisfying $M_{K+1}(v) = \bzero$ for all $v$, giving the router a learned ``off'' switch for with no dedicated parameters. The full \tide{} layer update is:

\begin{align}
  \balpha^\ell &= \softmax\!\bigl(W_r^\ell\,\tilde{n}^\ell\bigr) \in \R^{K+1},
  \label{eq:router}
\end{align}
\begin{equation}
  m^\ell(v) = \sum_{k=1}^{K+1} \alpha_k^\ell\, M_k(v),
  \qquad
  h^\ell = \tilde{h}^\ell + \FFN\!\bigl(\tilde{n}^\ell\bigr) + m^\ell(v).
  \label{eq:tide_model}
\end{equation}
where $W_r^\ell \in \R^{(K+1) \times d}$ is a per-layer learned weight matrix and $\sum_{k=1}^{K+1}\alpha_k^\ell = 1$, $\alpha_k^\ell > 0$ for all $k$. The memory vector $m^\ell(v)$ is added \textit{additively} and \textit{independently} of the FFN output: neither pathway interact with the other, preserving the residual stream's role as a shared communication channel \citep{elhage2021circuits}. Given that $\mathbf{M}$ is indexed by discrete token identity $v$, not by hidden state $h^\ell$, the memory contribution of each token is independent of contextual mixing at any depth.

\begin{figure}
  \centering
  \includegraphics[width=0.65\linewidth]{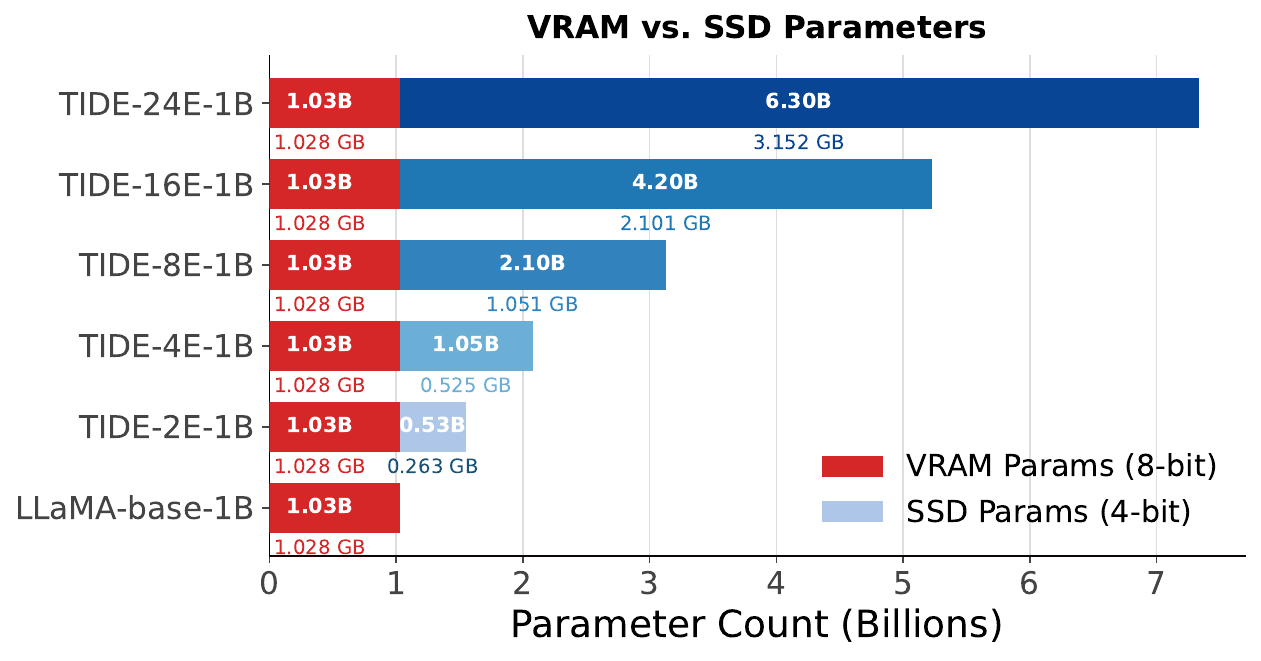}
  \caption{VRAM \& SSD parameter breakdown across LLaMA-\texttt{Base}-1B and \tide{}-1B model family with varying \memblock{} counts $K \in \{2, 4, 8, 16, 24\}$.}
  \label{fig:vram-ssd}
\end{figure}
\textbf{Computational and Memory Overhead:} In \tide{}, each $M_k(v) = \mathrm{RMSNorm}(E_k[v])$ is a single embedding lookup followed by RMSNorm and \textit{contributes no matrix multiplications}, so the per-layer overhead reduces to one $(K{+}1)$-way softmax router and a weighted sum of $d_b$-dimensional vectors. This is negligible relative to the baseline FFN. More importantly, every $E_k$ is indexed by discrete token identity $v$ independent of $h^\ell$, so once training completes the \embmem{} tables are static and can be 4-bit quantized (negligible performance impact) and offloaded to SSD for on-demand asynchronous prefetch augmented with appropriate caching mechanism. As Figure~\ref{fig:vram-ssd} shows, this maintains the \textit{effective VRAM footprint} of \tide{} similar as LLaMA-Base-1B level ($1.03$ GB in 8-bit) while the SSD footprint scales from $0$ to $3.152$ GB from $K{=}0$ to $K{=}24$. Additional details regarding inference overhead and \memblock{}s compression techniques can be found in Appendix \ref{app:inference_cost} and \ref{app:compression_tide}.

\subsection{TIDE: Theoretical Perspectives and Observations.}
\label{sec:theory}
 % ── 2.3.1 Asymptotic Generalisation ──────────────────────────────────────────
\subsubsection{Asymptotic Generalization to Standard Transformer.}
\label{sec:generalisation}
 
\begin{proposition}[Asymptotic Generalization]
\label{prop:gen}
Let $\mathcal{F}_{\mathrm{base}}$ denote the function class of standard transformers~\eqref{eq:base} and $\mathcal{F}_{\mathrm{TIDE}}$ the class of our proposed \tide{} models~\eqref{eq:tide_model}. For any $\epsilon > 0$, there exist finite router parameters $W_r^\ell$ such that
\[
  \norm{m^\ell(v)} < \epsilon \qquad \forall\, v \in \V,\;
  \ell \in \{1,\ldots,L\}.
\]
That is, $\mathcal{F}_{\mathrm{TIDE}}$ can approximate the standard transformer $\mathcal{F}_{\mathrm{base}}$ to an arbitrary precision.
\end{proposition}

\textbf{\textit{Proof sketch.}}
\renewcommand{\qedsymbol}{}
Since $M_{K+1}(v) = \bzero$, any weight assigned to the null bank contributes nothing to $m^\ell(v)$. By the softmax constraint, increasing the null logit $z_{K+1}^\ell$ jointly suppresses all active bank weights: $\sum_{k=1}^{K} \alpha_k^\ell = K/(K + e^{z_{K+1}^\ell}) \to 0$ as $z_{K+1}^\ell \to \infty$.

The triangle inequality then gives $\norm{m^\ell(v)} \leq (1 - \alpha_{K+1}^\ell) \cdot C \to 0$, where $C = \max_{v,k}\norm{M_k(v)} < \infty$. Setting $z_{K+1}^\ell = s^* = \log(K(C-\epsilon)/\epsilon)$ achieves $\norm{m^\ell(v)} < \epsilon$ at a finite parameter configuration. The full proof can be found in Appendix~\ref{app:proof}.

\subsubsection{TIDE's K-Pathway Gradient Amplification.}
\label{sec: rare_token_theory}
In section \ref{sec:grad_starvation} for the standard transformer, we discussed that the embedding $e_v$ of a rare token $v$ receives a non-zero gradient update only in steps where $v$ appears in the batch, yielding an expected cumulative squared gradient norm bounded by $\tau \cdot f_v \cdot BT \cdot G^2$. Our proposed \tide{}'s architecture provides a design advantage of $K$ independent \memblock{}s that enable $K$ distinct, parallel gradient pathways into each token's embedding tables on \emph{every} training step, regardless of how rarely it occurs in the corpus. We formalize the advantage as:

\begin{proposition}[$K$-Pathway Gradient Amplification]
  \label{prop:k_pathway}
   Let $\mathcal{L}_s$ denote the loss at step $s$ and let $e^{(k)}_v \in \R^{d_b}$ be the embedding of token $v$ in \memblock{} $k$. Under minibatch SGD, the total expected cumulative squared gradient norm across all $K$ embedding tables for token $v$ satisfies:
  \begin{equation}
    \E\!\left[\sum_{s=1}^{\tau} \sum_{k=1}^{K} \norm{\nabla_{e^{(k)}_v} \mathcal{L}_s}^2\right] \;\geq\; K \cdot \tau \cdot \kappa_v \cdot G^2_{\min},
    \label{eq:k_amplification}
  \end{equation}
  where $\kappa_v = 1 - (1 - f_v)^{BT} \approx f_v \cdot BT$ for small $f_v$, and $G^2_{\min} > 0$ is a lower bound on the per-step squared gradient norm conditioned on token $v$ appearing in the batch. Consequently, \tide{} provides a $K$-fold amplification of gradient signal relative to the standard single-embedding baseline.
\end{proposition}

\textbf{\textit{Proof sketch.}}
\renewcommand{\qedsymbol}{}
  Each \memblock{} $k$ maintains an independent embedding table $E_k \in \R^{|\mathcal{V}| \times d_b}$ with no parameter sharing across blocks\footnote{For the simplicity, we state the argument for a router over $K$ active banks.}. Within a forward pass during training, \memblock{} $k$'s output $M_k(v)$ is injected into every transformer layer $\ell$ via the routing weight $\alpha^\ell_k$, contributing to the residual stream and thereby to the loss. Since the $K$ blocks are independent, the event $\{v \in \text{batch}_s\}$ triggers gradient flow through all $K$ embedding tables simultaneously. Because router weights are strictly positive for finite logits, each table receives a non-degenerate gradient on every step that $v$ appears. Summing across blocks and applying the lower bound from Appendix~\ref{app:ratio} independently to each yields the $K$-fold amplification. Please see Appendix~\ref{app:k_amplification} for details.

\begin{figure}
    \centering
    \includegraphics[width=0.95\linewidth]{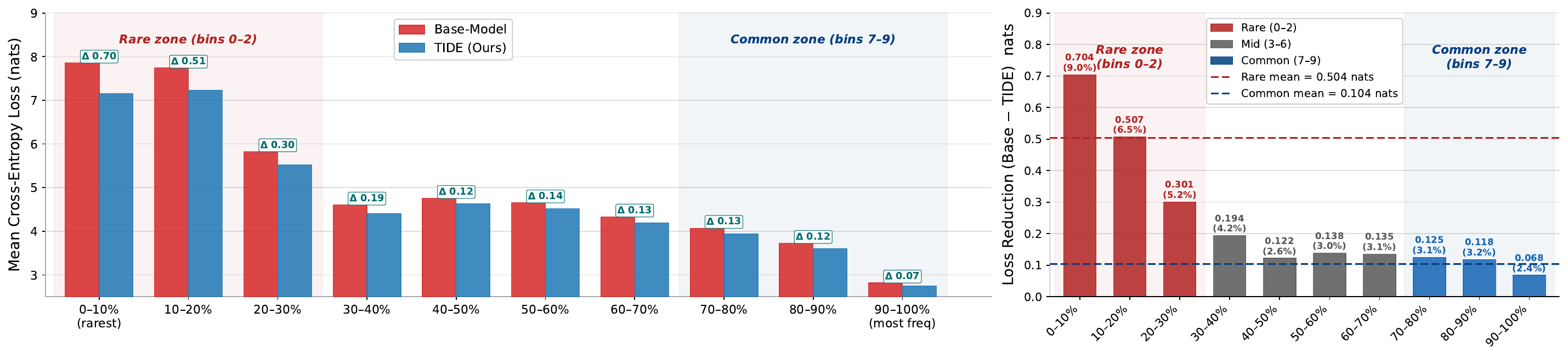}
    
    \caption{Mean validation cross-entropy loss per frequency decile of LLaMa-\texttt{Base}-1B and \tide{}-8E-1B trained with $200$B tokens. \tide{} strictly improves over the baseline on every decile with gain concentrated for \emph{rare} tokens and follows monotonically decreasing trend as \emph{rare} $>$ \emph{mid} $>$ \emph{common}.}
    
    \label{fig:loss_rare_tokens}
\end{figure}

$\bigstar$ \textbf{Empirical Investigation \textit{[Rare Tokens Benefits from \tide{}]}:} Figure \ref{fig:loss_rare_tokens}(a) illustrate the mean cross-entropy of LLaMa-\texttt{Base}-1B and \tide{}-8E-1B at the matched $200$B-token training budget across all $10$ token frequency deciles. Clearly, we can observe that \tide{} \emph{strictly outperforms} LLaMa-\texttt{Base}-1B on every decile, but the absolute performance gap is sharply asymmetric for \emph{rare vs. common} tokens. Per-decile loss reduction in Figure~\ref{fig:loss_rare_tokens}(b) decays \emph{monotonically} from $0.704$ nats ($9.0\%$ relative) on the rarest decile to $0.068$ nats ($2.4\%$) on the most frequent decile, yielding a $\sim$$4.8\times$ disparity in absolute gain between \emph{rare} and \emph{common} mean. This rare-skewed improvement profile is precisely the empirical signature provide support for $K$-fold gradient amplification to assist tokens where base embedding $E$ is gradient starved during training.

\subsubsection{Contextual Collapse and \tide{} $K$-\memblock{}s.}
\label{sec: contextual_collapse_theory}
In a standard transformer, FFN receives $h^{(\ell)}$ as input, and when $\norm{h^{(\ell)}_u - h^{(\ell)}_v} \leq \delta$ is small, Lipschitz continuity forces its outputs to remain close regardless of the weights chosen (see Section \ref{sec:contextual_collapse}). \tide{} architectural design permits to break this constraint since each \memblock{} is indexed by the discrete token identity $v$ unlike $h^{(\ell)}$, so its output carries no continuity obligation with respect to $\delta$. We formalize this observation as: 
\begin{proposition}[Memory Ensemble Resolves Collapsed Token Separation]
\label{prop:zero-collapse}
Let $(u,v)\in\mathcal{C}_\delta^{(\ell)}$ be a collapsed token pair satisfying
$\norm{h^{(\ell)}_u - h^{(\ell)}_v} \leq \delta$, and let $C > 0$ be any
target separation.
For any $K \geq 1$, there exist \embmem{} parameters $\{E_k\}_{k=1}^{K}$
such that:
\begin{equation}
  \norm{M_k(u) - M_k(v)} = C
  \label{eq:separation}
\end{equation}
regardless of $\delta = \norm{h^{(\ell)}_u - h^{(\ell)}_v}$ and independently of $L_{\mathrm{FFN}}$.
\end{proposition}

\textbf{\textit{Proof sketch.}}
Each \memblock{} output $M_k(v) = \mathrm{RMSNorm}(E_k[v])$, where $E_k[v]$ is the row of embedding table $E_k$ indexed by the discrete token identity $v$. The hidden state $h^{(\ell)}$ does not appear in this computation, so $M_k(u)$ and $M_k(v)$ depend only on their respective rows $E_k[u]$ and $E_k[v]$. Since these rows are \emph{separate, uncoupled parameters}, they can be assigned freely and independently for any token pair $(u, v)$, regardless of how small $\delta$ is. In particular, one can choose $E_k[u]$ and $E_k[v]$ such that the resulting RMSNorm outputs can achieve any prescribed separation $C > 0$, which satisfies~\eqref{eq:separation}. See Appendix \ref{app:collpase-proof} for the additional details. 

We would like to clarify that \tide{} does not attempt to fight the Lipschitz constraint of the FFN; it routes around it by exploiting a fundamentally different input signal during the training. Because $m^{(\ell)}(v) = \sum_{k=1}^{K}\alpha^{\ell}_k M_k(v)$ is re-injected in additive fashion at every transformer layer via independent per-layer router weights $W^{\ell}_r$, this token-discriminative signal persists throughout the residual stream, it enables effective separation at every layer $\ell$.

\begin{figure}
    \centering
    \includegraphics[width=0.99\linewidth]{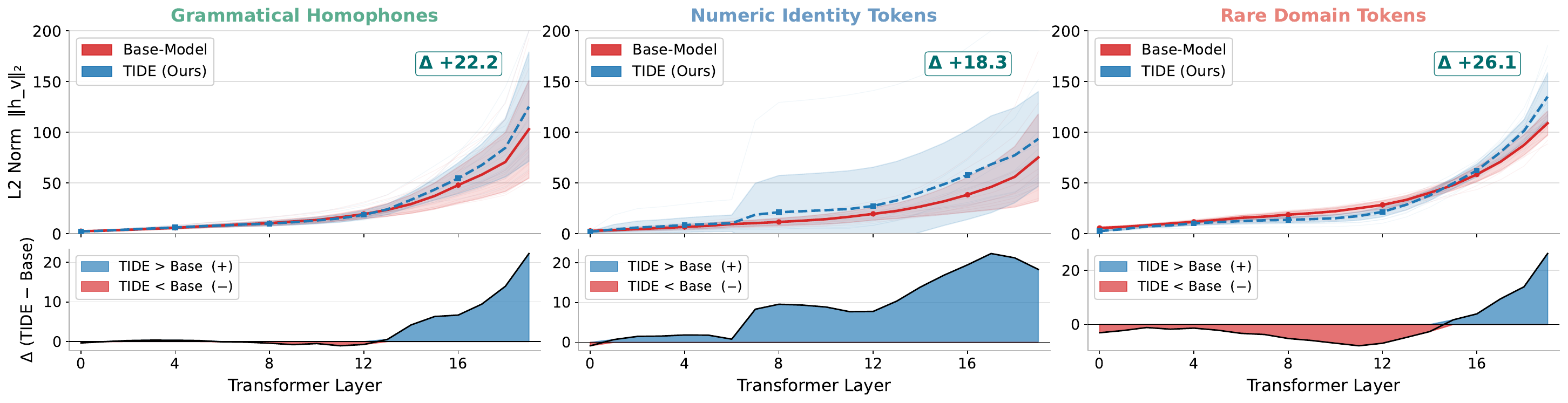}
    
    \caption{Layer-wise $\ell_2$ separation $\|h^{(\ell)}_u - h^{(\ell)}_v\|$ between hidden states of token pairs from the three example contextual collapse categories, averaged across $150$ template sentences.}
    
    \label{fig:collapse-base-tide}
\end{figure}

$\bigstar$ \textbf{ Empirical Investigation \textit{[Contextual Collapse is Moderated by \tide{}]}:} To empirically validate the contribution of additive pathways of \memblock{}s, we revisit the three example contextual collapse categories from Figure \ref{fig:collapse} (grammatical homophones, numeric identity tokens, rare domain tokens) and compare the layer-wise $\ell_2$ separation $\| h^{(\ell)}_u - h^{(\ell)}_v \|$ between LLaMa=\texttt{Base}-1B and \tide{} on the same template sentences. Figure~\ref{fig:collapse-base-tide} (top row) reports the mean $\ell_2$ norm averaged over all sampled token pairs in each category and bottom row reports the per-layer difference $\Delta = \|\cdot\|_{\tide{}} - \|\cdot\|_{\textsc{Base}}$. Across all three categories, we can clearly observe that \tide{}'s token-discriminative signal injection significantly increase in $\ell_2$ separation prominently from middle to terminal layers which are distant from base embedding $E$. Note that \textit{numerical tokens} which suffers acute collapse (Figure \ref{fig:collapse}), are the predominant beneficiary of the token identity injection throughout all layers.

\begin{table}[h]
\centering
\caption{Benchmark results for LLaMA-Base and TIDE variants at 750M, 1B and 3B parameter scales. PPL is LAMBDA perplexity (lower is better); BoolQ and LAMBADA use accuracy; all other columns use normalized accuracy (\%). Best results per scale are \textbf{bolded}.}

\label{tab:benchmark_results}
\resizebox{0.99\textwidth}{!}{%
\begin{tabular}{lc|cccccccc>{\columncolor{lightgray}}c}
\toprule
\textbf{Model} & \textbf{PPL~$\downarrow$} & \textbf{ARC-C} & \textbf{ARC-E} & \textbf{BoolQ} & \textbf{HellaSwag} & \textbf{LAMBADA} & \textbf{OBQA} & \textbf{PIQA} & \textbf{SciQ} & \textbf{Average} \\
\midrule
\multicolumn{11}{l}{\textit{750M Parameters}} \\
\midrule
LLaMA-Base       & 5.63          & 34.6          & 60.4          & \textbf{63.5} & 60.9          & 62.8          & 36.8          & 73.8          & 85.1          & 59.7 \\
TIDE-8E-750M     & \textbf{5.18} & \textbf{36.0} & \textbf{61.4} & 63.0          & \textbf{62.6} & \textbf{64.9} & \textbf{37.2} & \textbf{74.8} & \textbf{85.8} & \textbf{60.7} \\
\midrule
\multicolumn{11}{l}{\textit{1B Parameters}} \\
\midrule
LLaMA-Base       & 5.19          & 37.5          & 64.4          & 61.7          & 63.9          & 64.6          & 37.6          & 74.9          & 86.9          & 61.4 \\
TIDE-2E-1B       & 4.97          & 37.6          & 65.7 & 68.7          & 64.9          & 65.1          & 36.4          & 75.5          & 87.1          & 62.6 \\
TIDE-8E-1B       & 4.89          & 37.5          & 64.5          & 69.3          & 65.3 & 64.7          & \textbf{40.8} & 75.5          & 86.6          & 63.0 \\
TIDE-16E-1B      & 4.78 & 38.7 & 65.5          & \textbf{69.7} & 65.3& 65.7 & 37.8          & 75.9 & \textbf{87.9} & 63.3 \\
TIDE-24E-1B      & \textbf{4.60} & \textbf{38.9} & \textbf{66.3}          & 69.5 & \textbf{66.3} & \textbf{66.4} & 37.2          & \textbf{77.3} & 87.2 & \textbf{63.7} \\
\midrule
\multicolumn{11}{l}{\textit{3B Parameters}} \\
\midrule
LLaMA-Base   & 4.00    & 41.2 & 74.8          & 69.0          & 71.9 & 69.4          & 40.2          & 78.1 & \textbf{93.3} & 67.2 \\
TIDE-8E-3B   & \textbf{3.86}    & \textbf{44.3} & \textbf{75.5} & \textbf{72.3} & \textbf{72.2} & \textbf{70.2} & \textbf{40.6}          & \textbf{78.3}          & 93.2          & \textbf{68.3} \\
\bottomrule
\end{tabular}%
}
\end{table}

\begin{figure}
    \centering
    \includegraphics[width=0.75\linewidth]{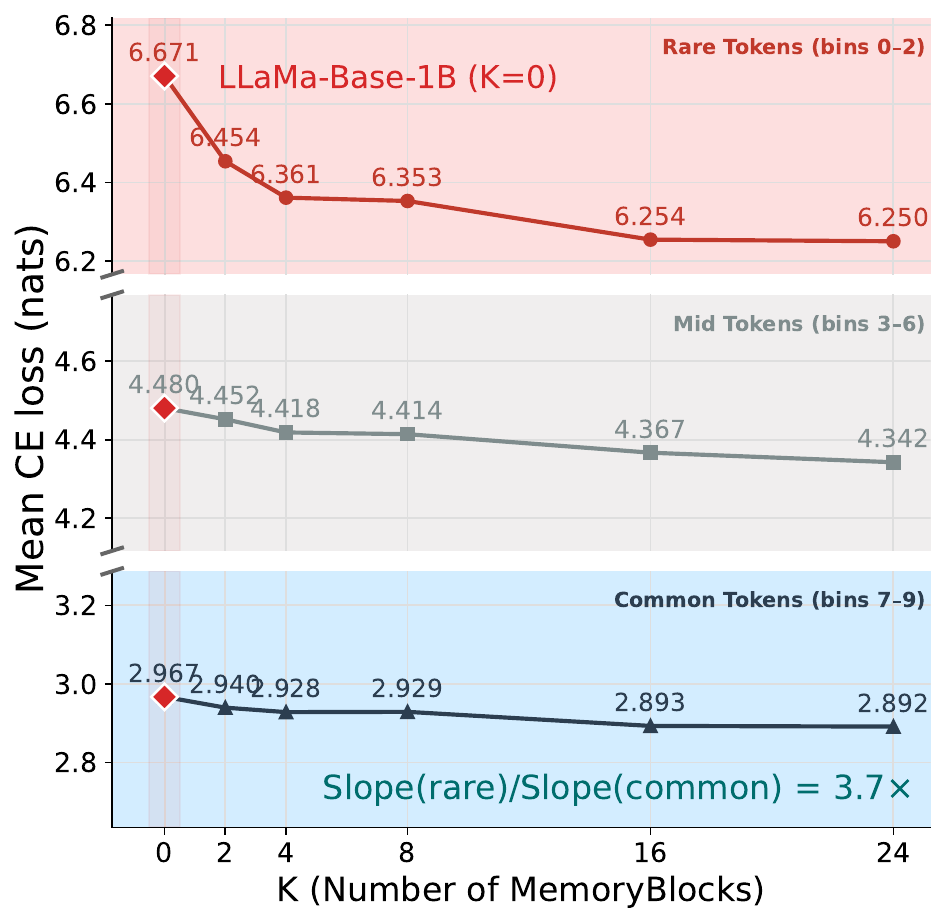}
    \caption{Mean cross entropy loss across \emph{rare, mid,} and \emph{common} tokens with increasing $K$ \memblock{}s.}
    \label{fig:catagoryloss}
\end{figure}
\section{Experiments and Ablation Studies}

\subsection{Performance Benchmarking of TIDE and Standard Transformer.}
\label{sec:performance_based_experiments}

\textbf{\ding{226} Perplexity and Training Dynamics:} \tide{} introduces a parallel additive \embmem{} pathways within conventional transformers to address the challenges associated with rare tokens and contextual collapse (Section \ref{sec: contextual_collapse_theory}, \ref{sec: rare_token_theory}). Here, we first investigate the influence of token-indexed memory's ability to improve the language modeling quality of standard transformers. Figure \ref{fig:ppl_based_tide_evaluation} presents the validation perplexity on three datasets - Wikitext \citep{merity2016pointer}, PubMed \citep{jin2019pubmedqa} and DCLM \citep{li2024datacomp} held-out corpora as a function of total training tokens for LLaMa-\texttt{Base}-1B and \tide{}-1B with $K \in \{2, 4, 8, 16, 24\}$. Clearly, each \tide{} variant \textit{strictly outperforms} LLaMa-\texttt{Base}-1B monotonically from $K=2$ to $K=24$ without saturation. Performance gap opens early during training where with 100B tokens \tide{} with merely $2-4$ \memblock{}s already matches the perplexity baseline reaches with 200B tokens, indicating that the additional gradient pathways translate to faster effective convergence.
% <Focus Figure 8>

\begin{figure}
    \centering
    \includegraphics[width=0.99\linewidth]{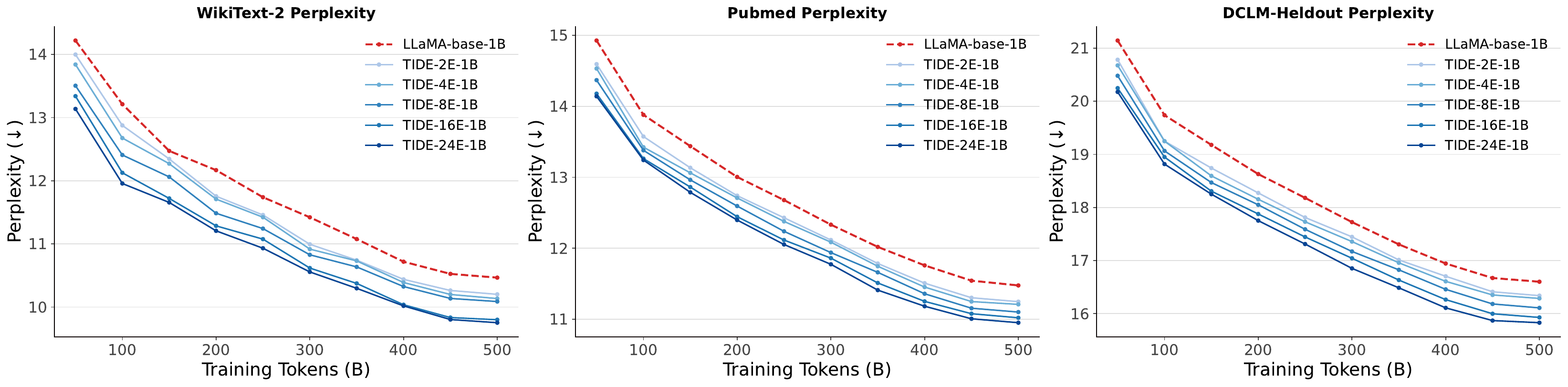}

    \caption{Wikitext-2, PubMed, DCLM validation PPL as a function of training tokens indicating monotonic improvement across increasing $K$ blocks of \tide{} variants.}

    \label{fig:ppl_based_tide_evaluation}
\end{figure}
\textbf{\ding{226} Influence of $K$ across across \emph{Rare, Mid, and Common} tokens:} While perplexity based evaluation provide an overall performance benefit of \tide{}, a natural question arises from Proposition \ref{prop:k_pathway} as: \textit{Do these $K$ \memblock{} pathways empirically benefits rare and common tokens equally, and does the marginal benefit of additional $K$ scale with token frequency?} Figure \ref{fig:catagoryloss} decomposes held-out cross-entropy loss across \emph{rare, mid, and common} bins as a function of increasing $K$. It can be observed that the absolute loss reduction over LLaMA-\texttt{Base}-1B is largest on rare tokens, moving from $K=0$ to $K=24$ reduces rare-bin loss from 6.671 to 6.250 nats ($-$0.421) compared to only $-$0.075 nats on common bin which is a 5.6$\times$ difference in absolute gain. In addition, the per-block marginal benefit (the slope of each curve) is 3.7$\times$ steeper on rare tokens than on common tokens, illustrating the alignment with  Section \ref{sec: rare_token_theory}. Note that even the smallest configuration $(K=2)$ can deliver $\sim$55\% of the total rare-token improvement obtained at $K=24$ (also reflected in PPL in Figure \ref{fig:ppl_based_tide_evaluation}), suggesting that the bulk of the benefit can be achieved with modest 2-4 memory blocks. 

\ding{226} \textbf{TIDE and Downstream Task Performance:} Table \ref{tab:benchmark_results} reports zero-shot accuracy on a suite of eight benchmarks (ARC-C, ARC-E, BoolQ, HellaSwag, LAMBADA, OBQA, PIQA, SciQ) across 750M, 1B, and 3B parameter scales of LLaMa-\texttt{Base} and \tide{} variants. Across all  settings, \tide{} variants consistently outperform standard transformer baselines, confirming the robustness of our proposed architecture. More specifically, at the 1B scale, \tide{} improves the average score from 61.4 (\texttt{Base}) to 63.7 ($K=24$), a $+$2.3\% absolute gain, with monotonic improvement in $K$ on the perplexity column and on six of eight downstream tasks.

\subsection{A Deeper Investigation of \memblock{}s and the NULL Bank.}
The performance results in Section \ref{sec:performance_based_experiments} establish that \tide{} $K$-pathways provide informative signal beyond the contextual residual stream. We now turn the investigation inward to understand the information stored across \memblock{}s and per-layer router dynamics after training.

\begin{figure}
    \centering
  
    \includegraphics[width=0.6\linewidth]{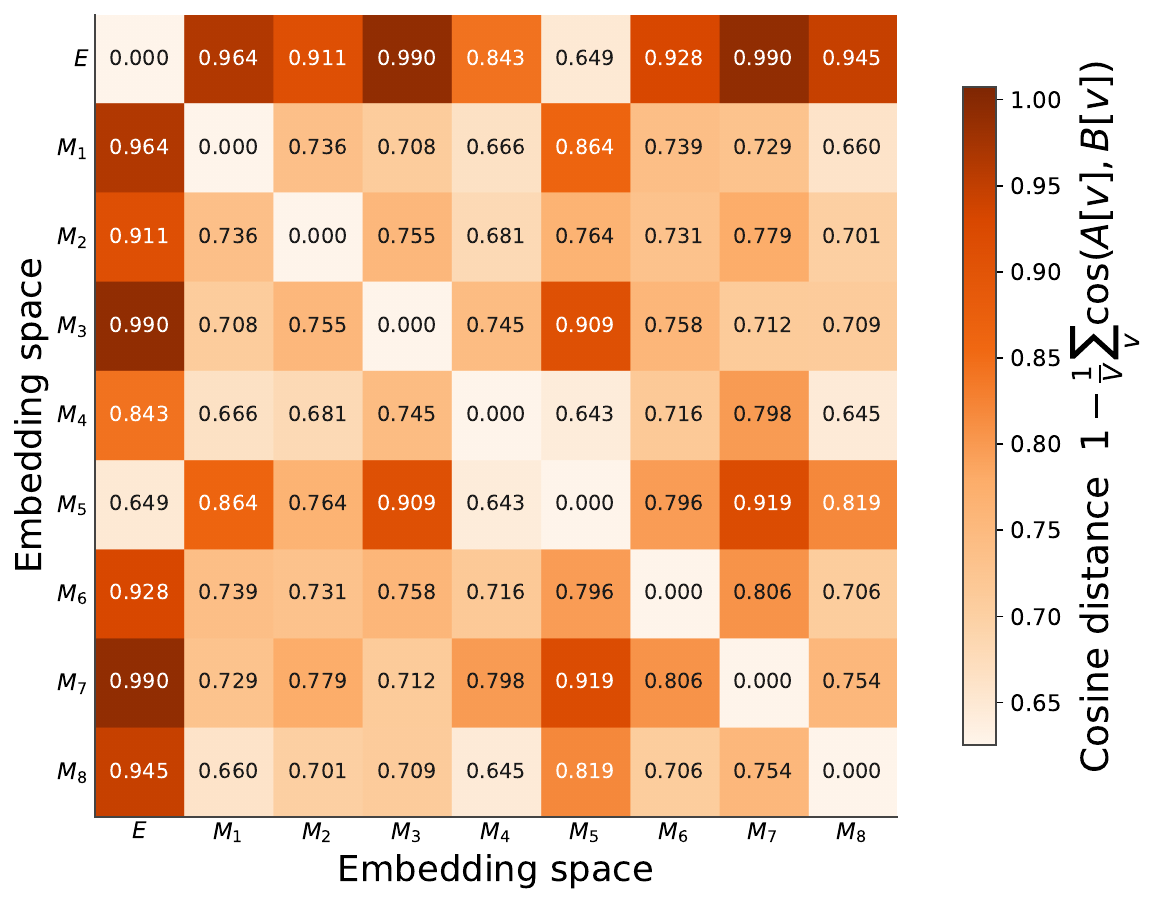}

    \caption{Mean cosine distance between primary embedding $E$ and 8 \memblock{}s.}
 
    \label{fig:embedding_distance}
\end{figure}
\textbf{\ding{226} Distance between Primary Embedding and \memblock{}s:} We first aim to investigate if \tide{}'s $K$-\memblock{}s converge to substantially distinct subspaces or collapse to replicate the base embedding $E$. Figure \ref{fig:embedding_distance} (a) reports mean cosine distance ($1 - \tfrac{1}{|\mathcal{V}|}\sum_v \cos(A[v], B[v])$ between every pair of embedding tables in \tide{}-8E-1B. We make two key observations: (a) every $M_k$ without any explicit diversity loss, is highly distant from $E$ (mean cosine distance to $E$ ranges from 0.65 to 0.99), confirming \memblock{}s do not replicate the input-embedding subspace but encode complementary token-identity signal; (b) inter-$M_k$ distance is relatively smaller, indicating convergence of $K$ blocks to overlapping but non-collapsed subspaces.

\textbf{\ding{226} Bin-wise Router Statistics for \memblock{}s and the NULL Bank:} In reference to our Proposition \ref{prop:gen} which state that TIDE asymptotically generalizes the standard transformer through the NULL bank, an empirical question persists: \textit{How does the router actually utilize this NULL bank, and does it do so in a token-aware manner?} Figure \ref{fig:catagoryloss_router_statistics} reports the mean routing weight $\bar\alpha_k$ allocated to each memory block $M_k$ and to the NULL bank from last layer stratified by frequency bin. We highlight two key observations: (a) the NULL-bank weight is \emph{monotonically non-decreasing in token frequency}: it rises from $\bar\alpha_{\text{NULL}} = 0.530$ on the rarest decile to $0.889$ on the most common decile. The router has therefore learned to open the gate and admit substantial memory-bank mass ($1 - \bar\alpha_{\text{NULL}} \approx 0.47$) for rarest tokens in comparison to common tokens; (b) the router weight is non-uniform across blocks where $M_5$ carries an outsized share on rare tokens ($\bar\alpha_5 \approx 0.31$) before collapsing to near-zero on common tokens wile $M_4$ specializes for mid-decile tokens illustrating that distinct banks specialize to distinct frequency regimes \emph{rather than redundantly co-firing}.

\begin{figure}
    \centering
    \begin{subfigure}[b]{0.48\linewidth}
        \centering
        \includegraphics[width=\linewidth]{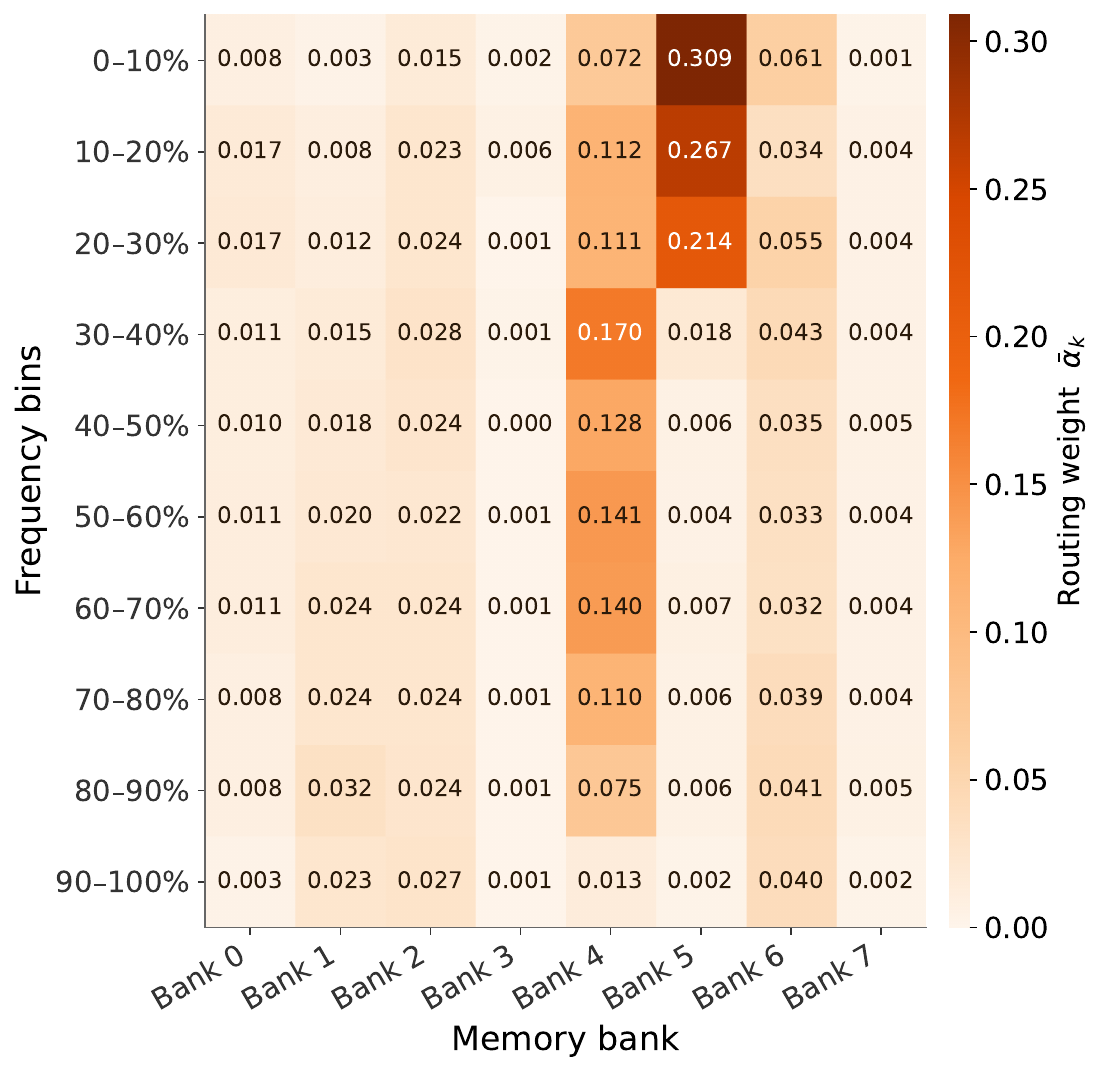}
        \label{fig:router_heatmap}
    \end{subfigure}
    \hfill
    \begin{subfigure}[b]{0.48\linewidth}
        \centering
        \includegraphics[width=\linewidth]{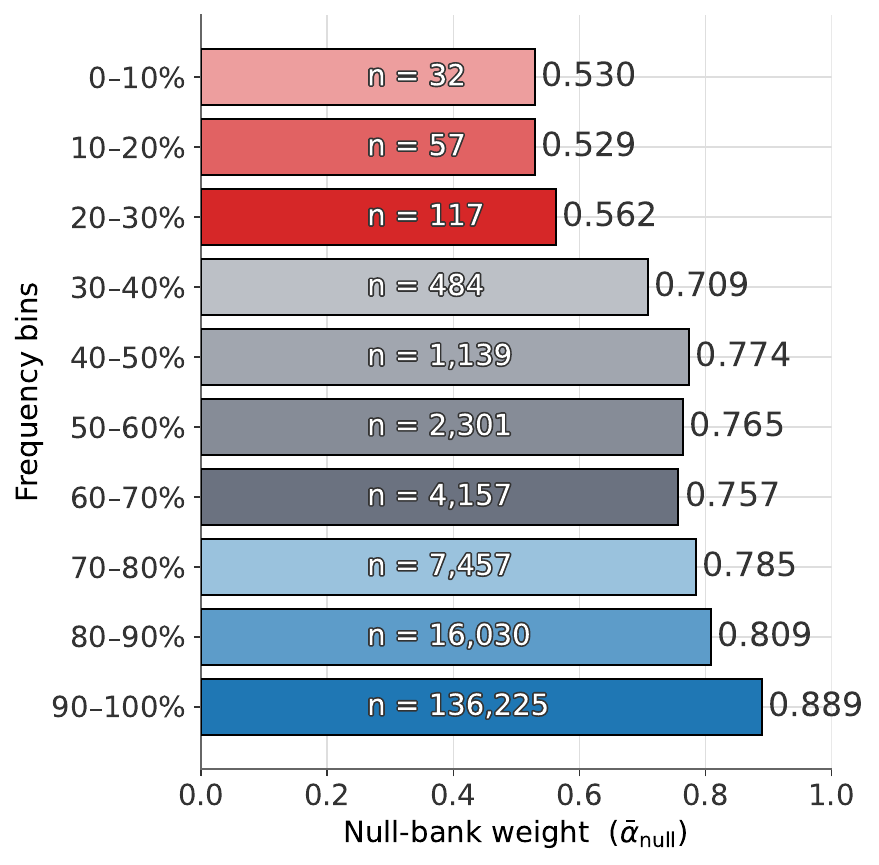}
        \label{fig:router_null}
    \end{subfigure}
    \caption{Bin-wise mean router weights $\bar\alpha_k$ across \memblock{}s (left) and the NULL bank (right), stratified by token frequency decile.}
    \label{fig:catagoryloss_router_statistics}
\end{figure}

\section{Conclusion}
\vspace{-2mm}
In this work, we propose \tide{}, a transformer architecture that addresses two empirically established failure modes of standard LLMs: gradient starvation of rare tokens and contextual collapse of semantically distinct tokens. We introduce \textsc{EmbeddingMemory}, an ensemble of $K$ independent \MB{}s that map token indices directly to semantic vectors, injected at every layer via a depth-conditioned router with a NULL bank. \tide{} provides each transformer layer with a persistent, token-specific signal that is immune to contextual collapse by construction. We \emph{theoretically} and \emph{empirically} establish the benefits of \tide{} in addressing the issues associated with single-token identity injection. With extensive experiments across different model scale, we found \tide{} consistently \emph{improve performance} across multiple language modeling and downstream tasks.

\newpage

\bibliography{references}
\bibliographystyle{plainnat}

%%%%%%%%%%%%%%%%%%%%%%%%%%%%%%%%%%%%%%%%%%%%%%%%%%%%%%%%%%%%%%%%%%%%%%%%%%%%%%%
%%%%%%%%%%%%%%%%%%%%%%%%%%%%%%%%%%%%%%%%%%%%%%%%%%%%%%%%%%%%%%%%%%%%%%%%%%%%%%%
% APPENDIX
%%%%%%%%%%%%%%%%%%%%%%%%%%%%%%%%%%%%%%%%%%%%%%%%%%%%%%%%%%%%%%%%%%%%%%%%%%%%%%%
%%%%%%%%%%%%%%%%%%%%%%%%%%%%%%%%%%%%%%%%%%%%%%%%%%%%%%%%%%%%%%%%%%%%%%%%%%%%%%%
\newpage
\appendix

\section{Background Work}

\subsection{Memory Augmented Architectures.}
Memory-augmented models are designed to expand a model’s effective parameter space without incurring large computational overhead. Early work on memory networks was introduced by \cite{weston2014memory}, and later extended with fully end-to-end trainable variants with \cite{sukhbaatar2015end}. Neural Turing Machines \citep{graves2014neural,Graves2016HybridCU} incorporate an external, trainable memory that works alongside other neural components to simulate a differentiable, trainable computing system. Product-key networks \citep{lample2019large} improve the efficiency and scalability of memory retrieval and propose a key-value memory layer that can scale to very large sizes while keeping
exact search on the key space. More recently, PEER \citep{he2024mixture} has advanced these ideas by replacing traditional vector-based memory values with rank-one matrices, linking memory-augmented architectures with mixture-of-experts models.

Accurate factual generation remains a critical objective for generative models, often evaluated using open-domain question answering benchmarks \citep{chen2017reading,chen-yih-2020-open} and other tasks requiring substantial knowledge \citep{petroni2021kilt}. Models that can effectively encode factual knowledge from training data are better equipped to provide correct responses to knowledge-intensive queries. While larger models generally demonstrate improved factual accuracy \citep{roberts2020much,brown2020language}, hallucination remains a persistent challenge. One effective approach for mitigating this issue is retrieval-augmented generation, which leverages external knowledge sources to improve factual consistency \citep{lewis2020retrieval,karpukhin2020dense,khandelwal2019generalization}. Several language models have incorporated text retrieval from the pretraining stage. REALM \citep{guu2020retrieval} augments a BERT model with one retrieval step to solve QA tasks. Retro \citep{borgeaud2022improving} enhances auto-regressive decoding with multiple rounds of retrieval, once per 64 tokens. The retrieved texts are injected through a two-layer encoder and then several cross-attention layers in the decoder. Retro++ \citep{wang2023shall} explores the scalability of Retro by reproducing Retro up to 9.5B parameters. Meanwhile, several models are adapted to retrieval in the finetuning stage. WebGPT \citep{nakano2021webgpt} learns to use search engine through imitation learning in a text-based web-browsing environment. Toolformer \citep{schick2023toolformer} performs decoding with multiple tools including search engine, and the finetuning data is labeled by the language model itself.

\subsection{Understanding Feed-Forward Networks in Transformers.}

Several studies have investigated the role of feed-forward networks (FFNs) in transformers, particularly their contribution to storing and retrieving knowledge learned during pretraining. \cite{geva2021transformer} demonstrated that FFNs can be interpreted as key–value memories that activate on specific lexical or semantic patterns, while follow-up work showed that FFNs promote vocabulary-level concepts during prediction \citep{geva2022transformer2}. Additional related analyses in embedding space further explored how FFN activations correspond to linguistic features and factual recall \citep{dar2023analyzing,nichani2024understanding}. Within their framework, the first layer acts as a pattern detector (``keys") while the second layer projects specific information into the residual stream (``values"). This modularity is evidenced by the identification of specific ``knowledge neurons" responsible for storing distinct facts. More broadly, the interpretation of neural networks as associative or persistent memory systems connects this line of work to earlier memory-augmented architectures \citep{sukhbaatar2019augmenting}. However, these analyses rely on contextualized residual activations and require extensive post-hoc mining of calibration data, making the inferred query space indirect and difficult to interpret. Furthermore, since FFNs operate exclusively on its contextualized residual stream, their ability to distinguish tokens is mathematically bottlenecked when distinct token appears in identical syntactic context which leads to the contextual collapse problem. Recently, MoLE \citep{jie2025mixture}, illustrates that in mixture-of-experts (MoEs), majority of experts can be trained directly with token-level input embeddings. Following the static routing concept, MemoryLLM \citep{jaiswal2026memoryllm} completely decouples FFNs from the contextual residual stream by directly training a layer-local and token-indexed embedding table to enhance interpretability and reduce compute. Concurrently, in the STEM \citep{sadhukhan2026stem} architecture, FFN is partially replaced to embedding table, with the substitution occurring at the up-projection layer. TIDE builds upon this token-level intuition but fundamentally diverges from standard FFNs. Instead of relying on contextual mixtures vulnerable to collapse, TIDE bypasses this entirely by injecting a context-free token identity directly into residual stream at every depth. 

\subsection{Advancements with Embedding and Modern LLMs.}

Standard transformer models rely on a single-injection assumption where token embeddings are looked up once at the input layer and subsequently faded out. Since language vocabularies strictly obey Zipf's law \citep{zipf1949human, pilgrim2021bias}, majority of tokens infrequently appear in the training corpus. Sub-word tokenization \citep{sennrich2016bpe} is introduced to mitigate out-of-vocabulary issue, yet they do not resolve the fundamental long-tail distribution of tokens, which continues to degrade the performance of contextualized embeddings on rare words. Under standard stochastic gradient descent, this skewed distribution leads to gradient starvation for rare tokens. Embedding sharing \citep{inan2017tying,ofir2017tying} attempt to stabilize training of embedding by tying input and output embedding weights, allowing input representations to benefit directly from the richer gradient signal of the pre-softmax layer. However, simply sharing parameters between the input and output layers does not structurally resolve the gradient starvation on low-frequency tokens. TIDE directly solve this gradient starvation bottleneck by utilizing independent memory blocks which lead to amplification of gradient signals to token representations, disproportionately benefiting rare tokens.

\section{Details of Frequency Bins Generated from Vocabulary}
\label{app:binning}
 
The WikiText-103 training split is tokenized with the LLaMA-3 tokenizer ($|\mathcal{V}| = 128{,}256$, sequence length $T = 2{,}048$), producing a token stream of $\sim$120M tokens over $\sim$58k sequences, of which 65,569 vocabulary entries appear at least once. Raw occurrence counts are then passed through a structural filter that removes BOS/EOS special tokens, pure-whitespace tokens, and non-alphanumeric punctuation. In total, 28 tokens ($0.04\%$ of observed types) are removed, leaving 65,541 token types for the binning procedure.

\begin{table}[h]
\centering
\caption{%
  Frequency decile bin reference for WikiText-103 with the LLaMA-3 BPE
  tokenizer (128K vocabulary) after structural filtering.
  Each bin contains ${\approx}6{,}554$ token types ranked by corpus frequency.
  Representative example tokens are drawn from each tier for semantic illustration.%
}
\label{tab:bin_reference}
\setlength{\tabcolsep}{5pt}
\renewcommand{\arraystretch}{1.22}
\begin{tabular}{clrll}
\toprule
\textbf{Bin} &
\textbf{Freq.\ range} &
\textbf{Types} &
\textbf{Description} &
\textbf{Role} \\
\midrule
\rowcolor{rarerow}
0 & 1--2       & 6,555 & Hapax \& near-hapax tokens    & \textcolor{rarecol}{\textbf{Rare}} \\
\rowcolor{rarerow}
1 & 2--6       & 6,554 & Domain-specific, rare names   & \textcolor{rarecol}{\textbf{Rare}} \\
\rowcolor{rarerow}
2 & 6--20      & 6,554 & Uncommon words, rare entities & \textcolor{rarecol}{\textbf{Rare}} \\
\midrule
\rowcolor{midrow}
3 & 20--61     & 6,554 & Infrequent content words      & \textcolor{midcol}{\textbf{Mid}} \\
\rowcolor{midrow}
4 & 61--133    & 6,554 & Occasional content words      & \textcolor{midcol}{\textbf{Mid}} \\
\rowcolor{midrow}
5 & 133--240   & 6,554 & Moderate-frequency words      & \textcolor{midcol}{\textbf{Mid}} \\
\rowcolor{midrow}
6 & 240--416   & 6,554 & Fairly common content words   & \textcolor{midcol}{\textbf{Mid}} \\
\midrule
\rowcolor{commonrow}
7 & 416--769        & 6,554 & Common function + content words & \textcolor{commoncol}{\textbf{Common}} \\
\rowcolor{commonrow}
8 & 769--1,856      & 6,554 & High-frequency content words    & \textcolor{commoncol}{\textbf{Common}} \\
\rowcolor{commonrow}
9 & 1,856--999,999  & 6,554 & Highest-frequency (below cap)   & \textcolor{commoncol}{\textbf{Common}} \\
\midrule
\rowcolor{rarerow}
\multicolumn{2}{l}{\textit{Rare examples (Bins 0--2)}} & \multicolumn{3}{l}{%
  \texttt{cefuroxime},\ \texttt{morgan},\ \texttt{Produto},\ \texttt{Teotihuacan},\
  \texttt{toujours}}  \\[1pt]
\rowcolor{midrow}
\multicolumn{2}{l}{\textit{Mid examples (Bins 3--6)}} & \multicolumn{3}{l}{%
  \texttt{volcano},\ \texttt{diocese},\ \texttt{battalion},\
  \texttt{peninsula},\ \texttt{sculptor},\ \texttt{harbour}} \\[1pt]
\rowcolor{commonrow}
\multicolumn{2}{l}{\textit{Common examples (Bins 7--9)}} & \multicolumn{3}{l}{%
  \texttt{there},\ \texttt{their},\ \texttt{also},\
  \texttt{however},\ \texttt{first},\ \texttt{time}} \\
\bottomrule
\end{tabular}
\end{table}

\paragraph{Decile assignment of Token Types:} The 65,541 cleaned types are sorted by ascending corpus frequency and partitioned into $B = 10$ equal-cardinality bins, each containing ${\approx}6{,}554$ types, with bin index assigned as:
\begin{equation}
  b(v) \;=\; \min\!\left(
    \left\lfloor
      \frac{\mathrm{rank}(v)}{|\mathcal{V}_{\mathrm{clean}}|} \cdot B
    \right\rfloor,\;
    B - 1
  \right).
  \label{eq:bin_assignment}
\end{equation}

Bin assignment is determined by rank alone while the absolute frequencies establish the ordering. Crucially, while every bin contains the same number of token \emph{types}, the bins account for vastly different shares of the training token \emph{stream} under Zipf's law. Throughout this paper, Bins $\{0-2\}$ are referred to as \underline{\textbf{rare}}, Bins $\{3-6\}$ as \underline{\textbf{mid-frequency}}, and Bins $\{7-9\}$ as \underline{\textbf{common}} tokens.

% \begin{wrapfigure}{r}{0.39\linewidth}
%     \centering
%     \vspace{-1.0em}
%     \includegraphics[width=\linewidth]{figures/rare_embedding_analysis_plot.pdf}
%     \vspace{-1.5em}
%     \caption{Cosine-Nearest Agreement between $E$ and memoryblocks $M_k$ for rare \& common tokens.}
%     \vspace{-1.0em}
%     \label{fig:jaccard_rare_common}
% \end{wrapfigure}

\section{Gradient Starvation Bound: Derivation of Equation~(\ref{eq:grad_starvation})}
\label{app:grad_starvation_proof}
The loss $\mathcal{L}_s$ depends on $e_v$ only through positions in $\mathrm{batch}_s$ that equal $v$.  If $v \notin \mathrm{batch}_s$, then $\partial \mathcal{L}_s / \partial e_v = \mathbf{0}$ exactly. Formally, we can write as:
\begin{equation}
  \nabla_{e_v}\mathcal{L}_s = \mathbf{0}
  \quad \text{whenever } v \notin \mathrm{batch}_s.
  \label{eq:sparsity}
\end{equation}
 
Let's define $X_s := \mathbf{1}[v \in \mathrm{batch}_s]$.  Since each of the $BT$ positions is drawn i.i.d.\ from $\{f_v\}$, the event $\{v \notin \mathrm{batch}_s\}$ requires all $BT$ draws to avoid $v$, each with probability $(1 - f_v)$.  Therefore:
\begin{equation}
  \Pr[v \in \mathrm{batch}_s]
  \;=\; \E[X_s]
  \;=\; 1 - (1-f_v)^{BT}
  \;\leq\; f_v \cdot B \cdot T,
  \label{eq:bernoulli}
\end{equation}
where the inequality applies the Bernoulli bound $(1 - f_v)^{BT} \geq 1 - f_v BT$, valid for $f_v \in [0,1]$.

Next, we bound the cumulative squared gradient using equation ~\eqref{eq:sparsity}, as $\norm{\nabla_{e_v}\mathcal{L}_s}^2 \leq G^2
\cdot X_s$.  Taking expectations and summing over $\tau$ steps:
\begin{align}
  \E\!\left[\sum_{s=1}^{\tau}\norm{\nabla_{e_v}\mathcal{L}_s}^2\right]
  &\leq \sum_{s=1}^{\tau} G^2 \cdot \E[X_s]
  \;\leq\; \tau \cdot f_v \cdot B \cdot T \cdot G^2,
\end{align}
which completes the derivation of ~\eqref{eq:grad_starvation}. 
 
\subsection{Understanding the Ratio of Gradient for Rare and Common Tokens}
\label{app:ratio}
Let $v$ be a rare token with $f_v = \varepsilon \ll 1/(BT)$ and $u$ a common token with $f_u \geq c > 0$. Using the aforementioned derivation, we have
 
\begin{equation}
  \E\!\left[\sum_{s=1}^{\tau}
    \|\nabla_{e_{v}}\cL_{s}\|^{2}\right]
  \;\leq\;
  \tau\cdot\varepsilon\cdot BT\cdot G^{2}.
  \label{eq:num_upper}
\end{equation}
 
To determine the lower bound for common frequency tokens, we assume a standard non-degeneracy condition: whenever token $u$ appears in batch~$s$, the per-step squared gradient norm satisfies $\|\nabla_{e_{u}}\cL_{s}\|^{2}\geq G_{\min}^{2}>0$ on the event $\{u\in\mathrm{batch}_s\}$. This holds throughout training whenever the cross-entropy loss has not been minimized on token $u$.
 
Defining $X_s^{(u)}:=\mathbf{1}[u\in\mathrm{batch}_s]$
and using $\|\nabla_{e_{u}}\cL_s\|^2\geq G_{\min}^2\cdot
X_s^{(u)}$, we take expectations and sum over $\tau$ steps:
\begin{equation}
  \E\!\left[\sum_{s=1}^{\tau}
    \|\nabla_{e_{u}}\cL_{s}\|^{2}\right]
  \;\geq\;
  \tau\cdot\Pr[u\in\mathrm{batch}]\cdot G_{\min}^{2}
  \;=\;
  \tau\bigl(1-(1-f_{u})^{BT}\bigr)G_{\min}^{2}.
  \label{eq:denom_exact}
\end{equation}
Since $f_{u}\geq c>0$ and $1-(1-x)^{n}$ is non-decreasing in $x$:
\begin{equation}
  1-(1-f_{u})^{BT}
  \;\geq\;
  1-(1-c)^{BT}
  \;=:\;\kappa\;>\;0,
  \label{eq:kappa_def}
\end{equation}
where $\kappa$ is a strictly positive constant depending only on $c$, $B$, $T$. Substituting it into \eqref{eq:denom_exact} gives:
\begin{equation}
  \E\!\left[\sum_{s=1}^{\tau}
    \|\nabla_{e_{u}}\cL_{s}\|^{2}\right]
  \;\geq\;
  \tau\,\kappa\,G_{\min}^{2}.
  \label{eq:denom_lower}
\end{equation}
 
To determine the ratio between rare and common tokens, we divide \eqref{eq:num_upper} by \eqref{eq:denom_lower}:
\begin{equation}
  \frac{\E\!\left[\displaystyle\sum_{s}
        \|\nabla_{e_{v}}\cL_{s}\|^{2}\right]}%
       {\E\!\left[\displaystyle\sum_{s}
        \|\nabla_{e_{u}}\cL_{s}\|^{2}\right]}
  \;\leq\;
  \frac{\tau\,\varepsilon\,BT\,G^{2}}%
       {\tau\,\kappa\,G_{\min}^{2}}
  \;=\;
  \frac{BT\cdot G^{2}}{\kappa\cdot G_{\min}^{2}}
  \cdot\,\varepsilon
  \;
  \label{eq:ratio_final}
\end{equation}
Since $BT$, $G^{2}$, and $G_{\min}^{2}$ are fixed positive
constants. By the first-order Taylor expansion for small $c$, we have:
\begin{equation}
  \kappa \;=\; 1-(1-c)^{BT} \;\approx\; c\cdot BT
  \;,
  \label{eq:kappa_order}
\end{equation}
so the prefactor satisfies:
\begin{equation}
  \frac{BT\cdot G^{2}}{\kappa\cdot G_{\min}^{2}}
  \;=\;
  \frac{BT\cdot G^{2}}{(c\cdot BT)\cdot G_{\min}^{2}}
  \;=\;
  \frac{G^{2}}{c\cdot G_{\min}^{2}}
  \;=\; O\!\left(\frac{1}{c}\right),
\end{equation}
and the full bound on the ratio is $O(\varepsilon/c)$, which completes the proof of~\eqref{eq:grad_ratio}.
 
\paragraph{Concrete evaluation on WikiText-103.}
With $\varepsilon=8.3\times10^{-9}$ (Bin-0 rare tokens, $n_v=1$),
$c=8.3\times10^{-3}$ (Bin-9 common token, $n_u\approx10^{6}$),
$B=8$, $T=2048$:
\begin{align}
  \kappa
    &\;=\; 1-(1-c)^{BT}
     \;\approx\; 1-e^{-c\,BT}
     \;=\; 1-e^{-136}
     \;\approx\; 1, \label{eq:kappa_concrete}\\[4pt]
  \frac{\varepsilon}{c}
    &\;=\; \frac{8.3\times10^{-9}}{8.3\times10^{-3}}
     \;=\; 10^{-6}. \label{eq:ratio_concrete}
\end{align}
Under the conservative assumption $G^{2}/G_{\min}^{2}=10$, the gradient signal accumulated by a Bin-0 hapax embedding is bounded above by $10^{-5}$ times that of a Bin-9 common token over the same training run, there exists a gradient disparity of five orders of magnitude.
 
% Using the aforementioned derivation for rare token $v$ with $f_v = \epsilon \ll 1/(BT)$
% and common token $u$ with $f_u \geq c > 0$, and taking the ratio:
% \begin{equation}
%   \frac{\E\bigl[\sum_s \norm{\nabla_{e_v}\mathcal{L}_s}^2\bigr]}
%        {\E\bigl[\sum_s \norm{\nabla_{e_u}\mathcal{L}_s}^2\bigr]}
%   \;\leq\; \frac{\tau \cdot f_v \cdot BT \cdot G^2}
%                {\tau \cdot f_u \cdot BT \cdot G^2}
%   \;=\; \frac{f_v}{f_u}
%   \;=\; \frac{\Pr[\text{position} = v]}{\Pr[\text{position} = u]}
%   \;\longrightarrow\; 0,
% \end{equation}
% as $\epsilon / c \to 0$, establishing~\eqref{eq:grad_ratio} illustrating gradient starvation problem for rare tokens. 

\section{Full Proof of Proposition~\ref{prop:ffn_lower_bound}}
\label{app:ffn_lower_bound_proof}
 
\textit{We prove that for any collapsed pair $(u,v)\in\mathcal{C}_\delta^{(\ell)}$ and any target function $g:\mathcal{V}\to\mathbb{R}^d$ with $\|g(u)-g(v)\| = C > L_{\mathrm{FFN}}\delta$, no setting of the FFN weights can approximate $g$ to error less than $(C-L_{\mathrm{FFN}}\delta)/2$ on both tokens simultaneously.}
 
By definition of the contextual collapse set, $(u,v)\in\mathcal{C}_\delta^{(\ell)}$ implies $\|h_u - h_v\|\le\delta$.  Since FFNs are Lipschitz in its hidden-state input \citep{virmaux2018lipschitz}, we have
\begin{equation}
  \|\mathrm{FFN}(h_u) - \mathrm{FFN}(h_v)\| \;\le\; L_{\mathrm{FFN}}\,\delta.
  \label{eq:lipschitz_bound}
\end{equation}
The FFN outputs for $u$ and $v$ must therefore lie within a ball of radius
$L_{\mathrm{FFN}}\delta$ of each other — they cannot be far apart, regardless
of how the weights $W_1, W_2$ are chosen.

Using triangle inequality on the target separation, we can write the target separation as:
\begin{align}
  C &= \|g(u) - g(v)\| \nonumber\\
    &= \|g(u) - \mathrm{FFN}(h_u)
       + \mathrm{FFN}(h_u) - \mathrm{FFN}(h_v)
       + \mathrm{FFN}(h_v) - g(v)\| \nonumber\\
    &\le \|g(u) - \mathrm{FFN}(h_u)\|
       + \|\mathrm{FFN}(h_u) - \mathrm{FFN}(h_v)\|
       + \|\mathrm{FFN}(h_v) - g(v)\|.
  \label{eq:triangle}
\end{align}
 
By substituting Lipschitz bound \eqref{eq:lipschitz_bound} into \eqref{eq:triangle}:
\begin{equation}
  C \;\le\;
  \|g(u) - \mathrm{FFN}(h_u)\|
  + L_{\mathrm{FFN}}\,\delta
  + \|g(v) - \mathrm{FFN}(h_v)\|.
  \label{eq:substituted}
\end{equation}
 
% \paragraph{Step 4: Isolate the approximation errors.}
Rearranging \eqref{eq:substituted} to isolate error terms:
\begin{equation}
  \|g(u) - \mathrm{FFN}(h_u)\| + \|g(v) - \mathrm{FFN}(h_v)\|
  \;\ge\; C - L_{\mathrm{FFN}}\,\delta.
  \label{eq:sum_error}
\end{equation}
The left-hand side is the sum of two non-negative approximation errors.  When
$C > L_{\mathrm{FFN}}\delta$, the right-hand side is strictly positive, so
at least one of the two errors is positive.  Specifically, since the
maximum of two non-negative numbers is at least half their sum:
\begin{equation}
  \max\bigl\{
    \|g(u) - \mathrm{FFN}(h_u)\|,\;
    \|g(v) - \mathrm{FFN}(h_v)\|
  \bigr\}
  \;\ge\; \frac{C - L_{\mathrm{FFN}}\,\delta}{2}.
  \label{eq:max_error}
\end{equation}
when $C > L_{\mathrm{FFN}}\,\delta$, the right side is strictly positive, which completes the proof.

Three quantities govern the bound among which none of which are under the FFN's control:
\begin{itemize}
  \item  \textbf{Target separation} ($C$): The required difference $\|g(u) - g(v)\|$ between the optimal representations of tokens $u$ and $v$.  This is determined by what the downstream task needs --- for example, the grammatical-class distance between \emph{``their''} (possessive determiner) and \emph{``there''} (locative adverb) is fixed by the language, not by the model's architecture.
 
  \item  \textbf{Input proximity} ($\delta$): The distance $\|h_u - h_v\|$ between the hidden states that the attention layer produces for $u$ and $v$.  When two tokens appear in nearly identical contexts with same surrounding words and same syntactic position - attention cannot distinguish them, and $\delta$ is small.  The FFN receives $h_u$ and $h_v$ as inputs; it cannot \emph{choose} to receive different inputs.
 
  \item \textbf{FFN change limit} ($L_{\mathrm{FFN}}$): The Lipschitz constant controls how rapidly the FFN output can change per unit change in input.  While $L_{\mathrm{FFN}}$ depends on the weights and could in principle be made large, doing so causes exploding gradients and training instability.  A large $L_{\mathrm{FFN}}$ amplifies the FFN's response to \emph{every} input perturbation which is not only limited to the gap between $h_u$ and $h_v$.  This sharply degrades performance on the majority of tokens whose hidden states are not collapsed.
\end{itemize}

\section{Full Proof of Proposition~\ref{prop:gen}}
\label{app:proof}
 
\textit{We prove that for any $\epsilon > 0$ there exist finite router parameters
$W_r^\ell$ such that $\norm{m^\ell(v)} < \epsilon$ for all $v \in \V$ and
$\ell \in \{1,\ldots,L\}$.}
 
Given $M_{K+1}(v) = \bzero$ by definition, for any router weight
$\alpha_{K+1}^\ell \in (0,1)$:
\[
  \alpha_{K+1}^\ell \cdot M_{K+1}(v) = \alpha_{K+1}^\ell \cdot \bzero = \bzero.
\]
This ensures that null bank contributes nothing to \tide{}. The memory term therefore simplifies to:
\[
  m^\ell(v) = \sum_{k=1}^{K+1} \alpha_k^\ell\, M_k(v)
            = \sum_{k=1}^{K} \alpha_k^\ell\, M_k(v).
\]
 
By the softmax constraint~\eqref{eq:router},
$\sum_{k=1}^{K} \alpha_k^\ell = 1 - \alpha_{K+1}^\ell$.
Applying the triangle inequality, we can bound the memory norm as:
\begin{equation}
  \norm{m^\ell(v)}
  \;\leq\; \sum_{k=1}^{K} \alpha_k^\ell\, \norm{M_k(v)}
  \;\leq\; \bigl(1 - \alpha_{K+1}^\ell\bigr) \cdot C,
  \label{eq:normbound}
\end{equation}
where $C = \max_{v \in \V,\, k \leq K} \norm{M_k(v)} < \infty$.
 
Next, we express the active weight sum in terms of the null logit. Upon setting $z_{K+1}^\ell = s > 0$ and $z_k^\ell = 0$ for all $k \leq K$. By softmax:
\begin{equation}
  1 - \alpha_{K+1}^\ell = \frac{K}{K + e^s}.
  \label{eq:activesum}
\end{equation}
As $s \to \infty$, $K/(K + e^s) \to 0$, so the total active bank weight
vanishes.
 
Substituting~\eqref{eq:activesum} into~\eqref{eq:normbound}:
$\norm{m^\ell(v)} \leq KC/(K + e^s)$.
For any $\epsilon \in (0, C)$, solving $KC/(K + e^s) = \epsilon$ gives:
\begin{equation}
  s^* = \log\!\left(\frac{K(C - \epsilon)}{\epsilon}\right).
  \label{eq:sstar}
\end{equation}
Now, we have: $KC/(K + e^{s^*}) = KC\epsilon/KC = \epsilon$. Therefore
$\norm{m^\ell(v)} \leq \epsilon$ uniformly over all $v$ and $\ell$ for any
$s \geq s^*$.
 
Note that, the threshold $s^*$ is finite for any $\epsilon \in (0,C)$. Setting $W_r^\ell$
so that $W_r^\ell \tilde{n}^\ell \approx s^*\,\mathbf{e}_{K+1}$ is a finite
parameter assignment under which $\norm{m^\ell(v)} < \epsilon$ for all $v$ and
$\ell$.
 
\paragraph{Additional Remark:}
From~\eqref{eq:activesum}, $\sum_{k=1}^K \alpha_k^\ell = K/(K +
e^{z_{K+1}^\ell})$ depends \textit{only} on the null logit $z_{K+1}^\ell$. A
single large null logit jointly suppresses all $K$ active banks through softmax
competition facilitating the suppression degree of freedom to one scalar, regardless of $K$.

\section{Full Proof of Proposition~\ref{prop:k_pathway}: K-Pathway Gradient Amplification}
\label{app:k_amplification}
 
For simplicity, Proposition~\ref{prop:k_pathway} is stated for a simplified router over $K$ active banks that excludes the null bank at slot $K{+}1$.
 
For a fixed token $v \in \mathcal{V}$ and let $e^{(k)}_v$ denote the row of embedding table $E_k$ corresponding to $v$, for $k = 1,\ldots,K$. Define $X_s := \mathbf{1}[v \in \text{batch}_s]$. As established in Appendix~\ref{app:grad_starvation_proof} with Bernoulli Bound, we have:
\begin{equation}
\E[X_s] = 1 - (1-f_v)^{BT} =: \kappa_v \leq f_v \cdot BT.
\end{equation} 
Given that the $K$ \memblock{}s have no shared parameters, the gradient with respect to $e^{(k)}_v$ is identically zero whenever $v \notin \text{batch}_s$, and otherwise:
\begin{equation}
  \nabla_{e^{(k)}_v} \mathcal{L}_s = \sum_{\ell=1}^{L} \alpha^\ell_k \cdot \frac{\partial \mathcal{L}_s}{\partial m^\ell(v)} \cdot \frac{\partial M_k(v)}{\partial e^{(k)}_v},
\end{equation}
where $\partial \mathcal{L}_s / \partial m^\ell(v)$ is the upstream gradient from layer $\ell$'s residual stream. Since $M_k(v)$ enters every layer, each block accumulates gradient contributions across all $L$ layers.
 
During training, whenever $v \in \text{batch}_s$ and the loss has not been minimized on token $v$, we assume the standard non-degeneracy condition: for each $k$, there exists at least one layer $\ell^*$ such that
\begin{equation}
\label{eqn:lower-bound-mk}
  \norm{\nabla_{e^{(k)}_v} \mathcal{L}_s}^2 \geq G^2_{\min} > 0 \quad \text{on the event } \{v \in \text{batch}_s\}.
\end{equation}
 
Since the $K$ blocks are independent and each satisfies Equation \ref{eqn:lower-bound-mk}, and $\nabla_{e^{(k)}_v}\mathcal{L}_s = \mathbf{0}$ exactly when $X_s = 0$ (token $v$ absent from batch), we have $\norm{\nabla_{e^{(k)}_v}\mathcal{L}_s}^2 \geq G^2_{\min} \cdot X_s$. Summing across $K$ blocks: 
\begin{equation}
  \sum_{k=1}^{K} \norm{\nabla_{e^{(k)}_v} \mathcal{L}_s}^2 \;\geq\; K \cdot G^2_{\min} \cdot X_s.
\end{equation}
Taking expectations and summing over $\tau$ steps completes the proof of \eqref{eq:k_amplification}:
\begin{equation}
  \E\!\left[\sum_{s=1}^{\tau} \sum_{k=1}^{K} \norm{\nabla_{e^{(k)}_v} \mathcal{L}_s}^2\right] \;\geq\; K \cdot \tau \cdot \kappa_v \cdot G^2_{\min}. 
\end{equation}
 
\textbf{Additional Remark:} The standard transformer upper bound gives $\E\!\left[\sum_s \norm{\nabla_{e_v} \mathcal{L}_s}^2\right] \leq \tau \cdot f_v \cdot BT \cdot G^2$ from Appendix~\ref{app:grad_starvation_proof}. The \tide{} lower bound is $K$ times the analogous single-block lower bound, confirming a $K$-fold amplification under the assumption $G^2/G^2_{\min} = O(1)$.

\section{Additional Details for Proposition~\ref{prop:zero-collapse}: Contextual Collapse and \tide{}'s $K$-\memblock{}s}
\label{app:collpase-proof}
\textit{Proposition~\ref{prop:zero-collapse} state that for any collapsed pair $(u,v)\in\mathcal{C}_\delta^{(\ell)}$ with $\norm{h^{(\ell)}_u - h^{(\ell)}_v} \leq \delta$, and any target separation $C > 0$, there exist \embmem{} parameters $\{E_k\}_{k=1}^{K}$ such that $\norm{M_k(u) - M_k(v)} = C$ for any $K \geq 1$.}
 
\medskip
From Equation \ref{eq:memblock}, we have $M_k(v) = \mathrm{RMSNorm}(E_k[v])$, where $E_k[v]$ is the row of $E_k \in \mathbb{R}^{|\mathcal{V}|\times d_b}$ selected by discrete index $v$. 

The hidden state $h^{(\ell)}_v$ does not appear in~\eqref{eq:memblock}, so $M_k(u)$ and $M_k(v)$ are independent of $\delta = \norm{h^{(\ell)}_u - h^{(\ell)}_v}$. This stands in direct contrast to the FFN, for which Lipschitz continuity forces for the $\norm{\mathrm{FFN}(h_u) - \mathrm{FFN}(h_v)} \leq L_{\mathrm{FFN}}\delta$ regardless of the weights chosen, bounding the output separation from above by $L_{\mathrm{FFN}}\delta$.

\medskip
Since rows $E_k[u]$ and $E_k[v]$ are uncoupled parameters, the table $E_k$ assigns one dedicated row per vocabulary entry with no joint constraint, so assigning one row places
no restriction on the other, regardless of $\delta$. $M_k(u)$
and $M_k(v)$ can be set to any two vectors in the output range of
$\mathrm{RMSNorm}$ independently.
In particular, for any target $C > 0$ there trivially exist row
assignments such that $\norm{M_k(u) - M_k(v)} = C$, regardless of
$\delta = \norm{h^{(\ell)}_u - h^{(\ell)}_v}$ and independently of
$L_{\mathrm{FFN}}$. 

\begin{figure}[h]
    \centering
    % \vspace{-1em}
    \includegraphics[width=0.99\linewidth]{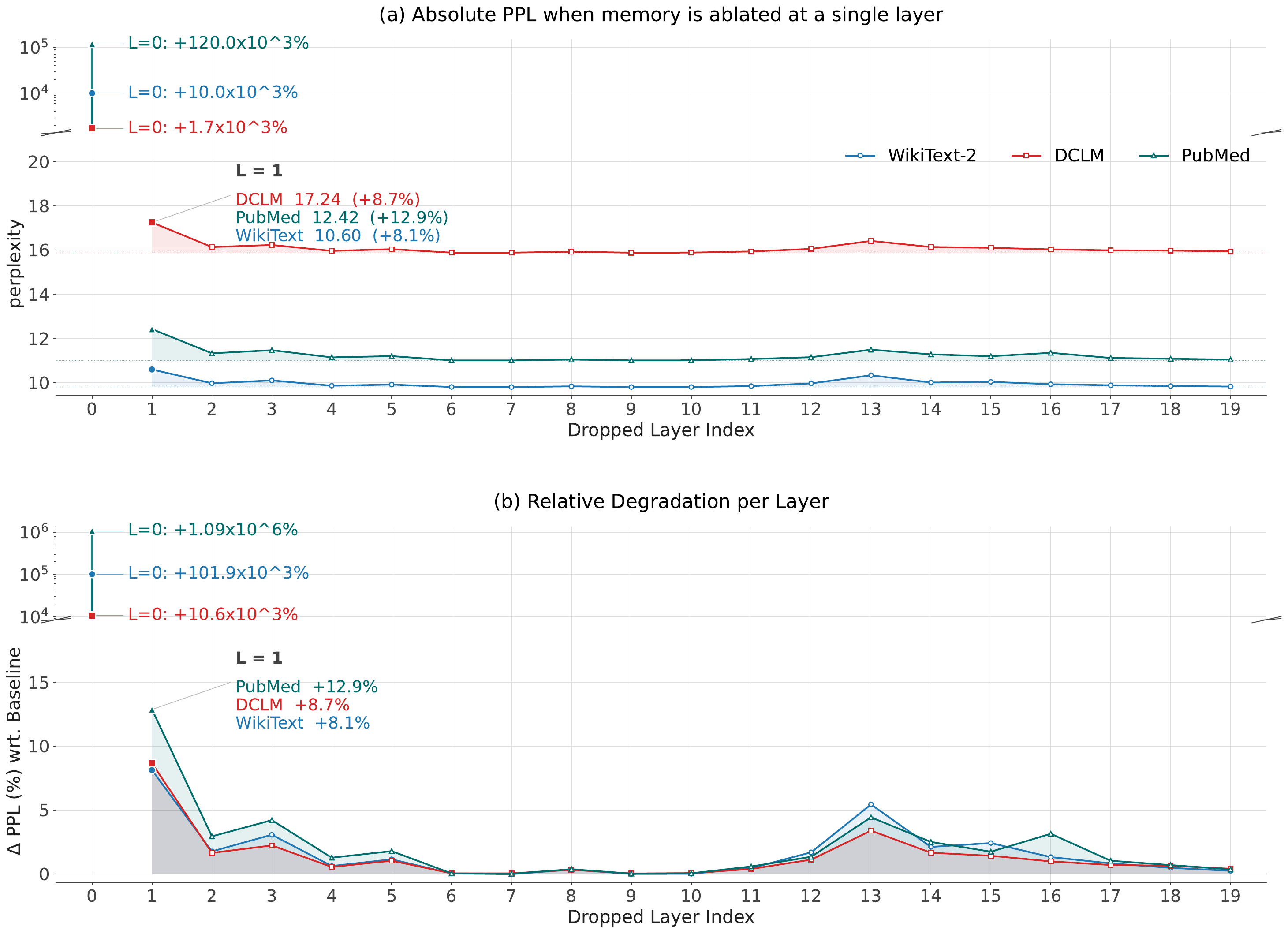}
    \caption{\textbf{Marginal contribution of each layer's memory injection in
\tide{}.} The routed \embmem{} output at one layer $\ell$ is zeroed
while all other layers retain their full pathway; we report perplexity on
WikiText-2, DCLM, and PubMed. Relatively higher degradation in PubMed performance also aligns with our rare-token problem.}
    % \vspace{-1em}
    \label{fig:layer_drop_ablation}
\end{figure}

\section{Understanding Layer-wise Contribution of \embmem{}}
\label{appendix:layerwise_contribution}

In \tide{}, we propose to add \embmem{} to the residual stream at every transformer
layer, a natural question arises: \emph{Is each layer's memory injection
equally important, or does it contribute disproportionately at certain
depths of the network?} To probe this, we sweep the layer index
$\ell \in \{0, 1, \ldots, L-1\}$ of our \tide-1B model and, at each
sweep point, replace the routed memory contribution at layer $\ell$ alone
with the zero vector while leaving the router and \memblock{}s pathway at every
other layer untouched. Concretely, the standard \tide{} forward pass at layer
$\ell$ is given as per Equation \ref{eq:tide_model} as:
\[
h^\ell = \tilde{h}^\ell + \FFN\!\bigl(\tilde{n}^\ell\bigr) + m^\ell(v).
\]
becomes $h^\ell = \tilde{h}^\ell + \FFN\!\bigl(\tilde{n}^\ell\bigr)$ for the single ablated layer  $\ell$, where
$\tilde{h}_\ell = \mathrm{RMSNorm}(h_\ell + \mathrm{Attn}_\ell(h_\ell))$
and $\mathrm{Mem}_\ell$ denotes the routed sum over the $K{=}24$ \memblock{}s.
We then evaluate perplexity on WikiText-2, DCLM, and PubMed, repeating for
every $\ell$ to obtain a per-layer marginal-contribution profile.

\textbf{Memory contribution is intermittent, not monotone:} Figure~\ref{fig:layer_drop_ablation} reveals that the contribution of \embmem{} is \textit{not uniform across depth} of the trained \tide{} checkpoint. Dropping layer~$0$ collapses the model entirely and the perplexity rises by more than $10^3\%$ on every dataset and reaches $1.09 \times 10^6\%$ on PubMed, indicating that the first memory injection carries an irreplaceable importance for the model. Layer~$1$ remains substantially load-bearing (\mbox{$+8.1$\%} to $+12.9\%$ across datasets), suggesting a brief consolidation period during which token-identity information is propagated into the residual stream. After this, degradation falls sharply: every layer in the contiguous range~$\ell \in [4, 12]$ lead to less than $\sim$2\% PPL cost.

We additionally observe a clear secondary peak at layer~$13$, where layer dropping costs new spike in performance degradation. We interpret this pattern as evidence that token-identity information injected by early \embmem{} layers \emph{persists} in the residual stream for several intermediate layers, during which any single memory contribution becomes redundant. Once this token-identity signal is consumed by the ongoing contextual computation, memory information is again required to refresh it at intermittent intervals.

\section{Understanding Decoding Cost with \memblock{}s}
\label{app:inference_cost}
In this section, we aim to investigate how does our proposed architecture \tide{}-1B empirically perform wrt. LLaMa-\texttt{Base}-1B in terms of decoding speed (ms/token). All experiments are reported using 1$\times$B200 GPU averaged across  5,120 generated tokens.   
\begin{table}[h]
    \centering
    \caption{Token Decoding estimated for \tide{}-1B variants in comparison to LLaMa-\texttt{Base}-1B transformer model.}
    \resizebox{\textwidth}{!}{
    \begin{tabular}{c|cccccc}
    \toprule
    & LLaMa-\texttt{Base}-1B & \tide{}-2E-1B & \tide{}-4E-1B & \tide{}-8E-1B &\tide{}-16E-1B & \tide{}-24E-1B \\
    \midrule
    Decoding Speed (ms/token) & 11.085 & 11.236 & 11.854 & 12.688 & 12.901 & 13.422\\
    \bottomrule
    \end{tabular}}
    
\end{table}

\section{Investigating Compression Strategies for \embmem{}}
\label{app:compression_tide}

Our proposed \tide{} architectures provide an opportunity to offload each static \memblock{}s within the \embmem{} to storage devices in resource-constrained settings with asynchronous pre-fetching. This leads to the question: \textit{How does the trade-off between VRAM and storage devices look like and what can be done to minimize \memblock{}s storage cost?}

We first estimate the total storage cost of \memblock{} as follows:
\begin{equation}
    \text{Storage Size} =  \mathrm{vocab\_size} \times \mathrm{num\_blocks} \times  \mathrm{hidden\_dim} \times \mathrm{bits\_per\_param }
\end{equation}
For our \tide{}-8E-1B model with 24 layers and 2048 hidden dimension which are trained with LLaMa-3.1 tokenizer with 128,256 vocabulary size, $\sim$4.2 GB of storage space is required for \embmem{} with F16 precision. To address our question, we performed a preliminary investigation\footnote{An effective novel compression technique for \embmem{} compression is out of the scope of this work. Our preliminary investigation reveals a high redundancy within the \memblock{}s and leaves sophisticated studies to capitalize these redundancies as future work.} of storage challenges of \embmem{} from \underline{\textbf{two}} different perspectives:
\begin{itemize}
    \item [D1.] Quantization of \embmem{}, and
    \item [D2.] Low Rank compression of individual \memblock{},
\end{itemize}

\subsection{Quantization of \embmem{}.}
\begin{table}[h]
    \centering
    \caption{Performance comparison of \tide{}-8E-1B with various low-precision of \memblock{}s.}
    \vspace{0.5em}
    \resizebox{0.6\textwidth}{!}{
    \begin{tabular}{cc|ccc}
    \toprule
    \textbf{Precision} & \textbf{Size (GB)} & \textbf{Wikitext-2 ($\downarrow$)} & \textbf{PubMed ($\downarrow$)} & \textbf{DCLM ($\downarrow$)}\\
    \midrule
    16-bit & 4.20 GB & 10.088 & 11.100 & 16.108\\
    8-bit & 2.10 GB & 10.089 & 11.100 & 16.113\\
    4-bit & 1.05 GB & 10.263 & 11.277 & 16.343\\
    \bottomrule
    \end{tabular}}
    
    \label{tab:placeholder}
\end{table}

\subsection{Low Rank Compression of Token-wise \memblock{}s.}

\label{appendix:low_rank_compression}

Several recent works \citep{li2023losparse, wang2023cuttlefish, kaushal2023lord} have
explored the low-rank characteristics associated with weights and gradients to
address storage demands and computational complexity linked to the large matrices
of LLMs. For our \tide-8E-1B model checkpoint with 8 \memblock{}s, each block
holds an embedding table $M_k \in \mathbb{R}^{|\mathcal{V}| \times d_b}$ that maps
a token index $v \in \mathcal{V}$ to a $d_b$-dimensional vector. A rank-$r$ SVD
decomposition of $M_k$ yields two matrices $U \in \mathbb{R}^{|\mathcal{V}| \times r}$
and $V \in \mathbb{R}^{r \times d_b}$, and rather than storing $M_k$ directly we
can store the factored representation $(U, V)$ provided $r$ is small enough that
the factored form has fewer parameters. We estimate the rank $r$, below which storage of $(U, V)$ will save space as follows:
\begin{equation}
(|\mathcal{V}| \cdot r) + (r \cdot d_b) \;\le\; |\mathcal{V}| \cdot d_b
\quad\Longrightarrow\quad
r \;\le\; \frac{|\mathcal{V}| \cdot d_b}{|\mathcal{V}| + d_b}.
\label{eq:rank-bound}
\end{equation}

For \tide-8E-1B with hidden dimension $d=2048$, bottleneck dimension $d_b=2048$,
and the \mbox{LLaMA-3.1} tokenizer of vocabulary size $|\mathcal{V}|=128{,}256$,
Equation~\ref{eq:rank-bound} gives $r \le 2015$, so any uniform rank reduction
of at least $\sim$2\% across the eight \memblock{}s is sufficient to reduce
storage relative to the dense parameterization.

\begin{figure}[h]
  \centering
  \includegraphics[width=\linewidth]{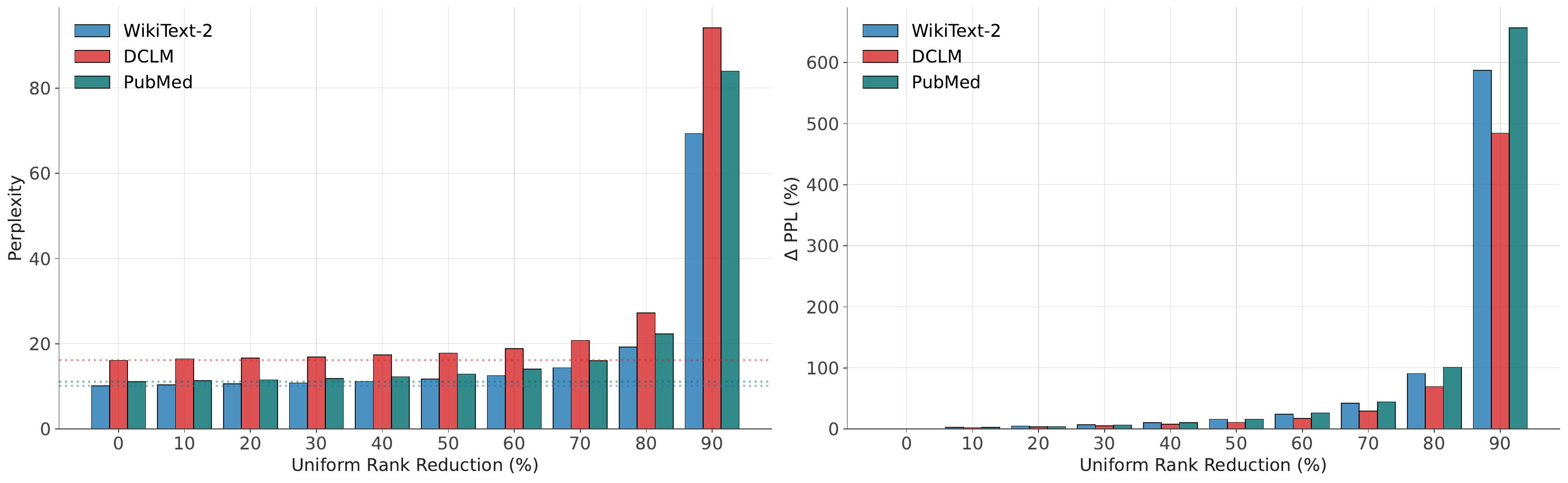}
  \caption{\textbf{Uniform rank reduction across all 8 \memblock{}s of
  \tide-8E-1B.} Each $M_k \in \mathbb{R}^{|\mathcal{V}| \times d_b}$ is replaced
  by its rank-$r$ SVD approximation with $r = \lceil(1-p) \cdot d_b\rceil$
  applied identically to every block.
  \textbf{(a)} Absolute perplexity on WikiText-2, DCLM, and PubMed; dotted
  horizontal lines mark each dataset's uncompressed baseline. \textbf{(b)}
  Relative degradation of perplexity which is largely flat through $\sim$30\% reduction, degrades gradually
  through $\sim$60\%, and rises sharply beyond 70\%.}
  \label{fig:rank-reduction}
\end{figure}

Figure~\ref{fig:rank-reduction}
sweeps the uniform reduction percentage from 0\% to 90\% in 10\% increments and
reports perplexity on WikiText-2, DCLM, and PubMed. At modest reductions (10--30\%, $r \in [1434, 1844]$), perplexity degradation remains almost flat while parameter count per \memblock{} drops to as little as 71\% of
the dense form. At moderate reductions (40--60\%, $r \in [820, 1229]$),
degradation grows around $\sim$10\% to $\sim$25\% but remains gradual. However, beyond 70\% reduction the curves bend sharply upward and the
relative \mbox{$\Delta$PPL} reaches 587\% on WikiText-2, 484\% on DCLM, and
657\% on PubMed. Our findings indicate that a
\tide{} \memblock{}s can be significantly compressed up to a $\sim$50\% rank reduction
($r=1024$, halving the per-block parameter count) at uniform with marginal drop in performance for practical purpose with limited resource availability. Provided the existence of non-uniform low-rank properties across different layers \citep{jaiswal2024galore}, we strongly believe that \embmem{} can be further compressed using non-uniform rank reduction techniques for relatively superior performance compared to uniform SVD.

\section{K-Nearest Neighbor Study of Base Embedding and \memblock{}s} 

To understand Base and \memblock{} embeddings from semantic perspective, we ask an interesting question: \textit{At individual token-level, do the \memblock{}s recover semantic neighbors that the primary embedding $E$ has failed to learn during training?} To probe this, we compute the per-token Jaccard overlap $J_k$ between the top-10 cosine neighbors across 200 randomly sampled rare and common token under $E$ and each $M_k$. 

In Figure \ref{fig:jaccard_rare_common}, we present the distribution of $J_k$ over tokens and find that rare-token boxes lie consistently below the common-token boxes for every $K$. It indicates that for common tokens, the neighbor sets agree more closely with $E$ but for rare tokens neighbor sets  are \textit{substantially disjoint} adding complementary and non-overlapping information about them. For example, in Table \ref{tab:knn_examples} with concrete examples: for a rare token \texttt{asynchronously}, $E$'s top-10 are dominated by adverbs ending in \{-ly\}; but each \memblock{}s contribute additional closely related technical information (\emph{e.g.} $M_2$ surfaces \emph{Asynchronous JavaScript and XML, callbacks}; $M_3$ surfaces \emph{defensively, securely}) in pursuit to enhance the semantic structure.

The rare-name query \texttt{fred} shows the same pattern: $E$ returns only first-name neighbors, while individual \memblock{}s additionally recover orthographic variants (\texttt{Fred}, \texttt{Frederick}, \texttt{Freddy}), tokenizer fragments (\texttt{Freder}), and cross-lingual variants (\texttt{Hans}, \texttt{Viktor}). The complementary-information picture therefore not a global statistical artifact but reflects genuine per-token specialization of \memblock{}s.

\begin{figure}[h]
    \centering
    \includegraphics[width=0.5\linewidth]{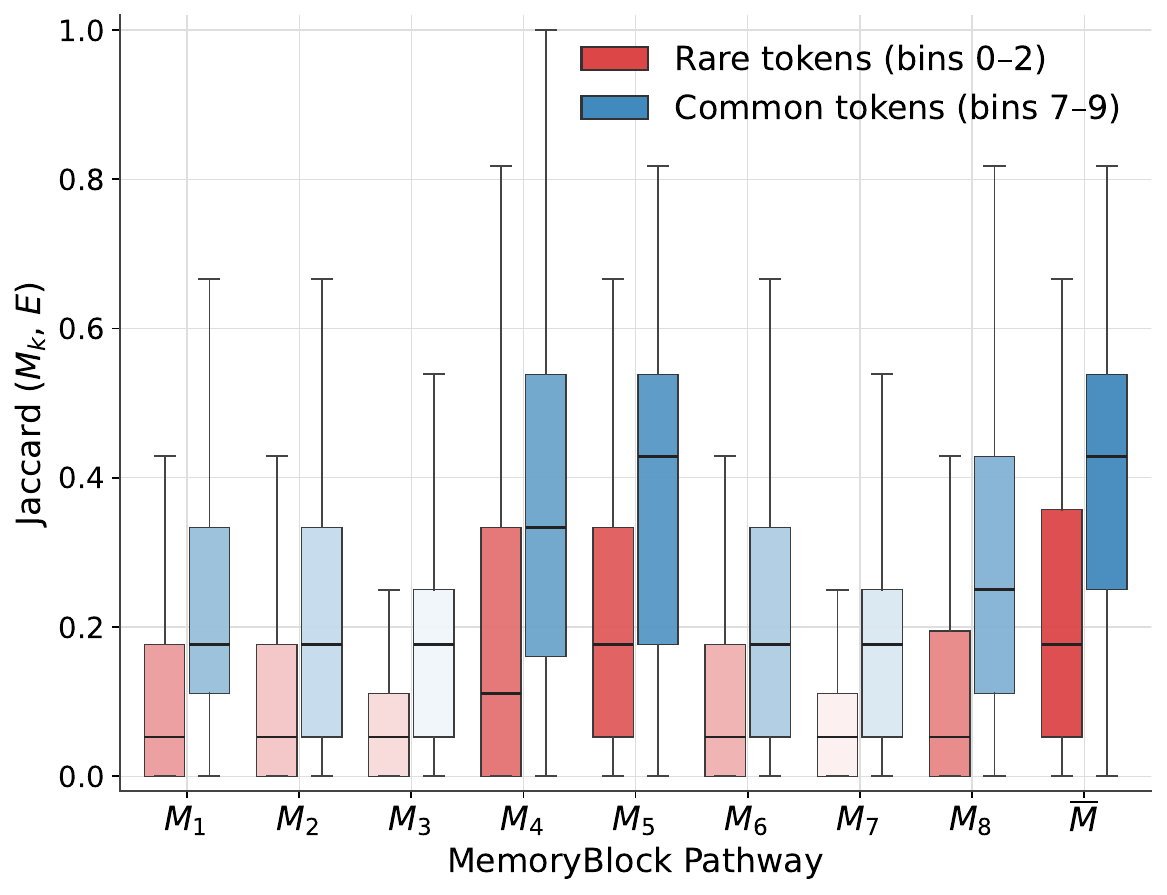}
    \caption{Cosine-Nearest Agreement between primary embedding $E$ and memory blocks $M_k$ across for \emph{rare} and \emph{common} tokens. Rare-token boxes lie consistently below common-token boxes indicating memory blocks pathways encode neighbor sets that are \textit{substantially disjoint} from $E$ for the rare tokens and add complementary new information to model.}
    \label{fig:jaccard_rare_common}
\end{figure}

\begin{table}
\centering
\caption{
Top-10 Cosine-Nearest neighbors of two rare query tokens in the primary embedding table $E$ and across 8 \memblock{} in \tide{}-8E-1B ($M_1, M_2, \ldots, M_8$) model checkpoints. Row backgrounds encode per-block Jaccard rank against the Base top-10 within each query, with \emph{darker} shades indicating \emph{higher} $J_k$ (more neighbor-set agreement with Base).
}
\label{tab:knn_examples}
\vspace{1.0em}
\begin{adjustbox}{max width=\textwidth}
\begin{tabular}{@{} c l l r @{}}
\toprule
\textbf{Query} & \textbf{Pathway} & \textbf{Top-10 nearest neighbors (cosine)} & $\mathbf{J_k}$ \\
\midrule
\rowcolor[HTML]{FADBD8}
\cellcolor{white} & Base $E$ & \texttt{asynchronous}, \texttt{ynchronously}, \texttt{sequentially}, \texttt{recursively}, \texttt{dynamically}, \texttt{ynchronous}, \texttt{concurrently}, \texttt{independently}, \texttt{horizontally}, \texttt{seamlessly} & -- \\
\rowcolor[HTML]{C2DBED}
\cellcolor{white} & $M_{1}$ & \texttt{asynchronous}, \texttt{ynchronously}, \texttt{synchronous}, \texttt{ynchronous}, \texttt{silently}, \texttt{globally}, \texttt{callback}, \texttt{concurrently}, \texttt{digitally}, \texttt{unsafe} & 0.25 \\
\rowcolor[HTML]{D3E5F2}
\cellcolor{white} & $M_{2}$ & \texttt{asynchronous}, \texttt{ynchronously}, \texttt{Asynchronous JavaScript and XML}, \texttt{callbacks}, \texttt{synchronous}, \texttt{ynchronous}, \texttt{LSD}, \texttt{hashtags}, \texttt{conditionally}, \texttt{breathable} & 0.18 \\
\rowcolor[HTML]{7FB3D9}
\cellcolor{white} & $M_{3}$ & \texttt{ynchronously}, \texttt{asynchronous}, \texttt{recursively}, \texttt{ynchronous}, \texttt{sequentially}, \texttt{defensively}, \texttt{resonate}, \texttt{dynamically}, \texttt{callbacks}, \texttt{securely} & 0.43 \\
\rowcolor[HTML]{B1D1E8}
\cellcolor{white} & $M_{4}$ & \texttt{asynchronous}, \texttt{ynchronously}, \texttt{concurrently}, \texttt{synchronous}, \texttt{ynchronous}, \texttt{dynamically}, \texttt{anonymously}, \texttt{externally}, \texttt{parallel}, \texttt{simultaneously} & 0.33 \\
\rowcolor[HTML]{A1C7E3}
\cellcolor{white} & $M_{5}$ & \texttt{asynchronous}, \texttt{ynchronously}, \texttt{simultaneously}, \texttt{recursively}, \texttt{spontaneously}, \texttt{sequentially}, \texttt{efficiently}, \texttt{tirelessly}, \texttt{separately}, \texttt{independently} & 0.33 \\
\rowcolor[HTML]{F5F9FC}
\cellcolor{white} & $M_{6}$ & \texttt{asynchronous}, \texttt{synchronous}, \texttt{Asynchronous JavaScript and XML}, \texttt{instantiated}, \texttt{serialize}, \texttt{scalable}, \texttt{serialized}, \texttt{ynchronously}, \texttt{caching} & 0.11 \\
\rowcolor[HTML]{E4EFF7}
\cellcolor{white} & $M_{7}$ & \texttt{asynchronous}, \texttt{ynchronously}, \texttt{manually}, \texttt{Premiere}, \texttt{optionally}, \texttt{reordered}, \texttt{indefinitely}, \texttt{RMS}, \texttt{factorial}, \texttt{charcoal} & 0.11 \\
\rowcolor[HTML]{90BDDE}
\cellcolor{white} \multirow{-9}{*}{\rotatebox[origin=c]{90}{\textbf{\texttt{asynchronously}}}} & $M_{8}$ & \texttt{asynchronous}, \texttt{synchronous}, \texttt{ynchronously}, \texttt{ynchronous}, \texttt{dynamically}, \texttt{synchronized}, \texttt{concurrently}, \texttt{electronically}, \texttt{simultaneously}, \texttt{automatically} & 0.33 \\
\midrule
\rowcolor[HTML]{FADBD8}
\cellcolor{white} & Base $E$ & \texttt{Fred}, \texttt{Fred}, \texttt{Larry}, \texttt{Roger}, \texttt{Doug}, \texttt{Charlie}, \texttt{Sean}, \texttt{jim}, \texttt{Mike}, \texttt{Jake} & -- \\
\rowcolor[HTML]{90BDDE}
\cellcolor{white} & $M_{1}$ & \texttt{Fred}, \texttt{Fred}, \texttt{joe}, \texttt{john}, \texttt{Ginny}, \texttt{Frederick}, \texttt{Doug}, \texttt{Lena}, \texttt{Woody}, \texttt{zar} & 0.18 \\
\rowcolor[HTML]{F5F9FC}
\cellcolor{white} & $M_{2}$ & \texttt{Fred}, \texttt{Fred}, \texttt{alf}, \texttt{Freder}, \texttt{Maggie}, \texttt{Carlo}, \texttt{Viktor}, \texttt{alan}, \texttt{Noel}, \texttt{Amit} & 0.11 \\
\rowcolor[HTML]{E4EFF7}
\cellcolor{white} & $M_{3}$ & \texttt{Fred}, \texttt{Fred}, \texttt{fred}, \texttt{mary}, \texttt{bob}, \texttt{Bob}, \texttt{Frederick}, \texttt{Bob}, \texttt{Freder}, \texttt{Herbert} & 0.11 \\
\rowcolor[HTML]{7FB3D9}
\cellcolor{white} & $M_{4}$ & \texttt{Fred}, \texttt{Fred}, \texttt{martin}, \texttt{Martin}, \texttt{Zack}, \texttt{Bernie}, \texttt{Frederick}, \texttt{alex}, \texttt{Charlie}, \texttt{Albert} & 0.18 \\
\rowcolor[HTML]{D3E5F2}
\cellcolor{white} & $M_{5}$ & \texttt{Fred}, \texttt{Fred}, \texttt{christ}, \texttt{Nora}, \texttt{brahim}, \texttt{Dani}, \texttt{tek}, \texttt{Enterprise}, \texttt{Yo}, \texttt{Practice} & 0.11 \\
\rowcolor[HTML]{C2DBED}
\cellcolor{white} & $M_{6}$ & \texttt{Fred}, \texttt{Fred}, \texttt{Katy}, \texttt{Hans}, \texttt{Ogre}, \texttt{Nel}, \texttt{Sing}, \texttt{Ian}, \texttt{Berk}, \texttt{Toby} & 0.11 \\
\rowcolor[HTML]{B1D1E8}
\cellcolor{white} & $M_{7}$ & \texttt{Fred}, \texttt{Fred}, \texttt{fred}, \texttt{Ron}, \texttt{Doug}, \texttt{COST}, \texttt{Todd}, \texttt{Evelyn}, \texttt{Lindsey}, \texttt{Apache} & 0.11 \\
\rowcolor[HTML]{A1C7E3}
\cellcolor{white} \multirow{-9}{*}{\rotatebox[origin=c]{90}{\textbf{\texttt{fred}}}} & $M_{8}$ & \texttt{Fred}, \texttt{Fred}, \texttt{fred}, \texttt{Frederick}, \texttt{Freddy}, \texttt{Ned}, \texttt{Freddie}, \texttt{Geoff}, \texttt{Fried}, \texttt{red} & 0.11 \\
\bottomrule
\end{tabular}
\end{adjustbox}
\end{table}

\section{Model Training Implementation Details}
\label{app:Implementation_config}

\begin{figure}[h]
    \centering
    \includegraphics[width=0.75\linewidth]{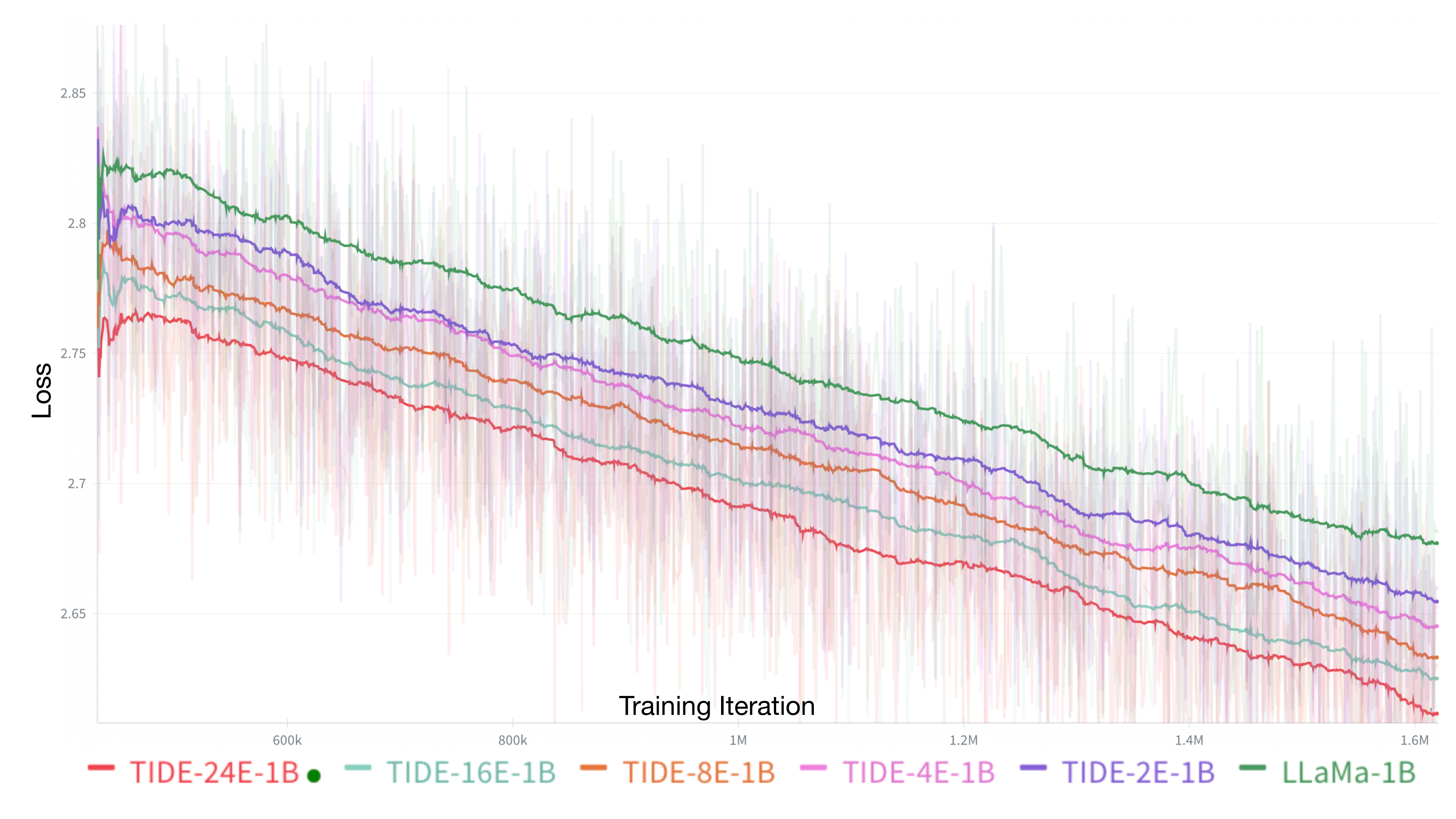}
    \caption{Training Loss Comparison of LLaMa-1B and TIDE-1B with 2, 4, 8, 16, and 24 \memblock{}s.}
    \label{fig:placeholder}
\end{figure}

\begin{table}[h]
    \centering
     \caption{Model training configurations for our $\mathrm{Base}$, \tide{} Models. All model checkpoints are trained within the paper adopt exactly same configuration for fair comparison.}
     \vspace{0.5em}
    \begin{tabular}{c|c|c}
    \toprule
    \textbf{Category} & \textbf{Key} & \textbf{Value} \\
    \midrule
        Common & Tokens Count & 400-500 Billion \\
        & Vocabulary size  &  128,256 \\
        & Tokenizer & meta-llama/Llama-3.1-8B \\
        & Dataset & mlfoundations/dclm-baseline-1.0 \\
        % & Padding Idx & 128,001 \\
        & Sequence Length & 2048 \\
        & Hidden Activation & SiLU \\
        \midrule
        Loss & Name & Cross Entropy \\
        & Z-loss & 1.0e-6 \\ 
        \midrule
        Optimizer & Name & Adam  \\
        & Weight Decay & 0.1\\
        & Beta1 & 0.9\\
        & Beta2 & 0.95 \\
        \midrule
        Schedular & Warmup Initial LR & 1e-06 \\
        & Warmup Iterations & 10000 \\
        & Type & Cosine \\
        & Max LR & 1.0e-04 \\
        & Min LR & 1.0e-05 \\
    \bottomrule
    \end{tabular}
   
    \label{tab:training_config}
\end{table}

\section{Limitations and Future Work}
\label{app:limitation}
While TIDE delivers consistent gains across model scales and downstream
tasks, we would like to acknowledge some limitations. \ding{182} \textit{Storage overhead:} Despite the EmbeddingMemory tables are static and quantization friendly, the SSD footprint still scales linearly with $K$. Deployments with strict storage budgets need to rely on the compression strategies discussed in Appendix~\ref{app:compression_tide}, along with conventional techniques \citep{jaiswal2023emergence,jaiswal2024ffn, li2023losparse,yin2023outlier} to reduce SSD overhead. \ding{183} Our experiments cover model scales from 750M to 3B parameters trained on 200--500B tokens of DCLM, with evaluations on WikiText, PubMed, DCLM, and eight zero-shot benchmarks. The benefits of \tide{} remain unexplored for longer training horizons, after instruction tuning or RLHF and is left to future work. \ding{184} Per-block router statistics and nearest-neighbor analyses (Appendix~\ref{tab:knn_examples}) suggest that distinct MemoryBlocks specialize to distinct frequency regimes, we do not provide a principled account of \emph{what} each block learns. A more fine-grained interpretability study of memoryblocks specialization is an important direction for future work.

\end{document}